\documentclass[letterpaper, 10pt, journal, twocolumn, final]{IEEEtran}
\usepackage{amsthm}
\usepackage{soul}
\usepackage[normalem]{ulem}
\usepackage[hidelinks]{hyperref}
\usepackage{bookmark}
\usepackage[latin1]{inputenc}
\usepackage[english]{babel}
\usepackage{mathtools}

\usepackage[dvipsnames]{xcolor}
\usepackage{amsfonts,amssymb,amsmath,color}
\usepackage[draft]{graphicx}
\usepackage{cite}
\usepackage{placeins}
\usepackage{graphicx}
\usepackage{algorithm}
\usepackage[]{algpseudocode}
\usepackage{todonotes}
\usepackage{balance}
\usepackage{tikz}
\usetikzlibrary{shapes,arrows,calc}

\newcommand\oprocendsymbol{\hbox{$\square$}}
\newcommand\oprocend{\relax\ifmmode\else\unskip\hfill\fi\oprocendsymbol}

\newtheorem{theorem}{Theorem}[section]

\newtheorem{lemma}[theorem]{Lemma}
\newtheorem{remark}[theorem]{Remark}

\newtheorem{assumption}[theorem]{Assumption}

\makeatletter
\newcommand{\StatexIndent}[1][3]{%
  \setlength\@tempdima{\algorithmicindent}%
  \Statex\hskip\dimexpr#1\@tempdima\relax}
\makeatother

\renewcommand{\lim}{\operatornamewithlimits{lim\vphantom{p}}}

\newcommand{\real}{{\mathbb{R}}}
\newcommand{\binset}{\{0,1\}}
\renewcommand{\natural}{{\mathbb{N}}}

\newcommand{\until}[1]{\{1,\ldots,#1\}} 

\newcommand{\map}[3]{#1: #2 \rightarrow #3}

\newcommand{\subj}{\textnormal{subj. to}}
 
\newcommand{\nbrs}{\mathcal{N}}
\newcommand{\conv}[1]{\textnormal{conv}(#1)}

\newcommand{\smallsum}{\textstyle\sum\limits}

\newcommand{\1}{\mathbf{1}}

\newcommand{\neigh}{\ell}

\renewcommand{\Xi}{X_i^\textsc{mi}}

\newcommand{\yend}{y^\textsc{end}}
\newcommand{\zend}{z^\textsc{end}}
\newcommand{\xend}{x^\textsc{end}}
\newcommand{\Bend}{B^\textsc{end}}
\newcommand{\Qend}{Q^\textsc{end}}

\newcommand{\cz}{z^\textsc{conv}}
\newcommand{\cy}{y^\textsc{conv}}
\newcommand{\ty}{\tilde{y}}
\newcommand{\tx}{\tilde{x}}
\newcommand{\tB}{\tilde{B}}
\newcommand{\tQ}{\tilde{Q}}

\newcommand{\bx}{\bar{x}}
\newcommand{\bB}{\bar{B}}
\newcommand{\bQ}{\bar{Q}}

\newcommand{\EE}{\mathcal{E}}

\newcommand{\GG}{\mathcal{G}}

\newcommand{\OO}{\mathcal{O}}

\newcommand{\eqdef}{\coloneqq}

\newcommand{\innbrs}{\mathcal{N}}
\newcommand{\Nag}{N}
\newcommand{\Pset}{P}
\newcommand{\Dset}{D}
\newcommand{\Rset}{R}

\newcommand{\Vset}{V_A}
\newcommand{\agents}{\mathbb{I}}

\usepackage{acronym}
\acrodef{PDVRP}{\emph{Pickup-and-Delivery Vehicle Routing Problem}}
\usepackage{cuted}

\graphicspath{{figs/}}

\def \algname/{Distributed Resource Allocation for PDVRP}

\begin{document}

\title{\LARGE Multi-Robot Pickup and Delivery via Distributed Resource Allocation}

\author{Andrea Camisa*,
  Andrea Testa*,
  Giuseppe Notarstefano
  \thanks{A. Camisa, A. Testa and G. Notarstefano are with the Department of Electrical, 
  Electronic and Information Engineering, University of Bologna, Bologna, Italy. 
  \texttt{\{a.camisa, a.testa, giuseppe.notarstefano\}@unibo.it}.
  This result is part of a project that has received funding from the European 
  Research Council (ERC) under the European Union's Horizon 2020 research 
  and innovation programme (grant agreement No 638992 - OPT4SMART).
  }%
  \thanks{* These authors contributed equally to this work.}%
}

\maketitle

\begin{strip}\leavevmode\kern15pt
\begin{minipage}{\dimexpr\linewidth-30pt\relax}
{\vspace{-2.5cm}
\bf \textcopyright 2022 IEEE. Personal use of this material is permitted.  Permission from IEEE must be obtained for all other uses, in any current or future media, including reprinting/republishing this material for advertising or promotional purposes, creating new collective works, for resale or redistribution to servers or lists, or reuse of any copyrighted component of this work in other works.}
\end{minipage}
\end{strip}

\begin{abstract}
 In this paper, we consider a large-scale instance of the classical
  \ac{PDVRP} that must be solved by a network of mobile cooperating robots.
  Robots must self-coordinate and self-allocate a set of pickup/delivery tasks
  while minimizing a given cost figure. This results in a large, challenging Mixed-Integer
  Linear Problem that must be cooperatively solved %
  without a central coordinator.
  We propose a distributed algorithm based on a primal decomposition approach that
  provides a feasible solution to the problem in finite time.
  An interesting feature of the proposed scheme is that each robot computes only
  its own block of solution, thereby
  preserving privacy of sensible information. The algorithm also exhibits
  attractive scalability properties that guarantee solvability of the problem even
  in large networks.
  To the best of our knowledge, this is the first attempt to provide a scalable
  distributed solution to the problem.
  The algorithm is first tested through Gazebo simulations on a ROS~2 platform,
  highlighting the effectiveness of the proposed solution.
  Finally, experiments on a real testbed with
  a team of ground and aerial robots are provided.
\end{abstract} %


\section{Introduction}
\label{sec:intro}
The Pickup-and-Delivery Vehicle Routing Problem (PDVRP)
is one of the most studied combinatorial 
optimization problems. The interest is mainly motivated by the practical 
relevance in real-world applications such as such as battery
exchange in robotic networks,~\cite{kamra2017combinatorial}, pickup and
delivery in warehouses,~\cite{ham2019drone}, task scheduling~\cite{gombolay2018fast}
and delivery with precedence constraints~\cite{bai2019efficient}. 
The \ac{PDVRP} is known to be an $\mathcal{NP}$-Hard optimization 
problem, and can be solved to optimality only for small instances.
In a PDVRP, a group of vehicles has to fulfill a set of transportation requests.
Requests consist of picking up goods at some locations and delivering them to other
locations. The problem then consists of determining minimal length paths
such that all the requests are satisfied.
To achieve this, we can assign to each location a label ``P'' (pickup)
or ``D'' (delivery), and then define a graph of all possible paths that can
be traveled by vehicles. In Figure~\ref{fig:pdvrp}, we show an example scenario
with two vehicles, two pickups and two deliveries.
\begin{figure}[htbp]\centering
  \includegraphics[scale=.92]{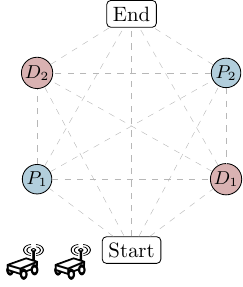}\hfill
  \includegraphics[scale=.92]{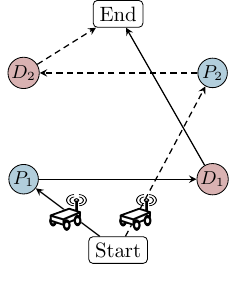}
  \caption{Example PDVRP scenario. Left: vehicles begin from the ``start'' node and must
  end at the terminal node. There are two pickup requests $P_1$ and $P_2$
  and two associated deliveries $D_1$ and $D_2$. Right: the optimal path consists of
  the first vehicle traveling through $P_1$ and $D_1$ and the second vehicle
  traveling through $P_2$ and $D_2$.}
  \label{fig:pdvrp}
\end{figure}
Note that, in order to have a well-defined graph, vehicles
must start from an initial node representing the initial position
(which may be different for each of them) and have to reach a target node.
This can be also ``virtual'' in the sense that, once the last delivery position has
been reached, the vehicle stops there and waits for further instructions.
With respect to classical Vehicle Routing Problems, the PDVRP introduces a set of additional variables and constraints that make the problem harder to solve. In particular, precedence constraints must ensure that the pickup of a good is performed before its delivery. Moreover, vehicles have capacity constraints that must be satisfied throughout the mission. These additional constraints are based on additional real variables that are not included in classical routing problems.
In order to determine the optimal path, one typically formulate 
an associated optimization problem and solves it to optimality.
Throughout the paper, we consider a general version of this optimization problem for which we propose a distributed algorithm, i.e., an algorithm that robots can run in a peer-to-peer fashion without a central coordinator. In this distributed setting, communication among robots occurs according to a given graph and it is not a design parameter.
As the distributed computation paradigm requires agents to share computation instead of data, it is particularly recommended in contexts where the local problem data (such as final assignment of tasks, vehicle capacity, cost of tasks, etc.) has to be maintained private as e.g. in military applications
or futuristic smart cities with robots belonging to different users.

\subsection{Related Work}

Several algorithms have been proposed to solve the problem to (sub)optimality
in centralized settings, we refer the reader to the surveys~\cite{toth2002vehicle,
parragh2008survey,pillac2013review,ritzinger2016survey}
for a comprehensive list of these methods in static and dynamic settings.
Online, dynamic approaches have been proposed, see, e.g.,~\cite{coltin2013online},
taking into account collision avoidance constraints among robotic agents,~\cite{liu2019task}.

In order to overcome the drawbacks of centralized approaches, as, e.g., the high
computational complexity of the problem, a branch of literature analyzes schemes based
on master-slave or communication-less architectures.
In master-slave approaches, implementations of the parallel auction based algorithm
are among the most used strategies, see, e.g.,~\cite{fauadi2013intelligent,heap2013repeated}.
As for communication-less approaches, we refer the reader to~\cite{arsie2009efficient}
for a dynamic pickup and delivery application. 
In~\cite{soeanu2011decentralized} authors address a multi-depot multi-split vehicle routing problem, 
where computing nodes apply a local heuristic based on a stochastic gradient descent and then exchange 
the local solutions in order to select the best one.

Few works address the solution of vehicle routing problems in peer-to-peer networks.
Indeed, the majority of the approaches in literature address
approximate problems or special cases of the pickup-and-delivery.
Authors in~\cite{chopra2017distributed,settimi2013subgradient,
burger2012distributed} solve unimodular task assignment problems by means of convex
optimization techniques. Other works are instead based on the solution of mixed-integer problems
with a simplified structure. 
In particular,~\cite{testa2020generalized,luo2015distributed} address generalized
assignment problems, where the order of execution of the tasks is not relevant. %
In~\cite{testa2019distributed}, authors address an %
assignment problem where agent-to-task paths are evaluated offline and capacity
constraints are not considered.
The distributed auction-based approach, see~\cite{talebpour2019adaptive,buckman2019partial}
for recent applications, is often used to address task allocation problems.
These approaches do not address more complex models with, e.g.,
load demands and execution time of tasks.
As for distributed approaches to vehicle routing problems,
in the works~\cite{pavone2009stochastic,pavone2010adaptive,bullo2011dynamic}, authors propose distributed schemes based on
Voronoi partitions for stochastic dynamic vehicle routing problems with time windows and  customer impatience.
Authors in~\cite{farinelli2020decentralized} propose a distributed algorithm where
agents communicate according to cyclic graph and perform operations one at a time.
In~\cite{abbatecola2018distributed}, a distributed scheme in which agents iteratively
solve graph partitioning problems is proposed, and is shown to converge to a suboptimal
solution of a dynamic vehicle routing problem.
A multi-depot vehicle routing problem is considered instead in~\cite{saleh2012mechanism},
where the problem is modeled as game in which customers and depots find a
feasible allocation by means of an auction game.

Since the PDVRP is a Mixed-Integer Linear Program (MILP), let us recall
some recent works addressing MILPs in distributed frameworks.
A distributed cutting-plane approach converging in finite time with arbitrary precision
to a solution of the MILP is proposed in~\cite{testa2019distributed}.
However, this approach requires each agent to compute
the entire solution of the problem and is not suited for multi-robot PDVRPs.
The work in~\cite{falsone2018distributed} addresses large-scale MILPs
by a dual decomposition approach, while in~\cite{camisa2019milp} the authors
propose a primal-decomposition algorithm with
low suboptimality bounds.
However, to the best of our knowledge, there are no works in the literature
specifically tailored for pickup-and-delivery problems.

\subsection{Contributions}
The contributions of the present paper are as follows. We formalize the Multi-vehicle
Pickup-and-Delivery Vehicle Routing Problem as a large-scale distributed optimization
problem with local and coupling constraints.
The considered formulation of the Pickup-and-delivery problem is general
and encompasses scenarios where each robot can only perform a subset of the tasks.
We propose a distributed algorithm
for the fast computation of a feasible solution to the problem. The algorithm is based on
a distributed resource allocation approach and requires only local computation
and communication among neighboring robots with no external coordination.
The proposed algorithm is scalable in the sense that the amount
of computation performed by each robot is independent of the network size.
Moreover, robots do not exchange private information with each other
such as vehicle capacity or the computed routes.
We formally prove finite-time feasibility of the solution computed by the algorithm
and we provide guidelines for practical implementation.
We first provide numerical computations on a ROS~2 set-up in which the \ac{PDVRP} is
solved with our algorithm and the dynamics of the vehicles
is realistically simulated through Gazebo.
Although the problem is $\mathcal{NP}$-hard,
the simulations highlight that our distributed algorithm is able to compute
``good-quality'' solutions within a few (fixed) iterations.
Then, we show results of real experiments performed on the ROS~2
testbed with TurtleBot3 ground robots and Crazyflie2 aerial robots.
We now highlight the main differences with other related approaches.
In~\cite{talebpour2019adaptive,buckman2019partial}, authors consider vehicle routing
problems with upper-bounds in the number of tasks a robot can serve. Dynamic variations of these set-up are considered in~\cite{pavone2009stochastic,pavone2010adaptive,bullo2011dynamic}. Our approach instead takes into account more general capacity constraints. Moreover,
we consider precedence constraints among tasks. Authors in~\cite{farinelli2020decentralized} 
propose a scheme for pickup-and-delivery problems in which vehicles execute the steps of the algorithm
once at a time and communicate according to a ring graph with a shared token. 
In our protocol instead agents can perform the optimization steps concurrently, can communicate according to general connected graphs and do not  need a shared memory.
Authors in~\cite{abbatecola2018distributed,saleh2012mechanism} consider the set-up in which vehicles start their mission from different depots, pick-up resources at given locations and then come back to their initial depot. In our set-up instead vehicles have to pick-up resources and deliver them to other stations before going to final depots.

The paper is organized as follows. In Section~\ref{sec:setup}, the problem statement
together with the needed notation and preliminaries is provided.
The distributed algorithm is provided in Section~\ref{sec:distributed_algorithm},
while Section~\ref{sec:analysis} encloses the theoretical analysis.
We describe the Gazebo simulations in Section~\ref{sec:simulations}
and, finally, we provide the experimental results in Section~\ref{sec:experiments}.

\section{Pickup-and-Delivery Vehicle Routing Problem}
\label{sec:setup}

In this section, we provide a mathematical formulation of the optimization
problem studied in the paper together with a thorough description of its structure.

\subsection{Optimization Problem Formulation}
\label{sec:optimization_problem}
We consider a scenario in which $N$ robots, indexed by $\agents \eqdef \until{N}$,
have to serve the transportation requests.
We denote by $\Pset \eqdef \until{|P|}$ the index set of pickup requests and
by $\Dset \eqdef \{|P+1|, \ldots, 2|P|\}$ the index set of delivery demands (with $\Pset \cap \Dset = \emptyset$).
To each pickup location $j\in P$ is associated a delivery location $j\in D$ (with $|P|=|D|$),
so that both the requests must be served by the same robot.
To ease the notation, we also define a set $R \eqdef P \cup D$ %
of \emph{all} the transportation requests (independently of their pickup/delivery nature).
Each request $j\in R$ is characterized by a service time $d_j\geq 0$, which is the
time needed to perform the pickup or delivery operation.
Within each request, it is also associated a load $q_j\in\real$, which is positive
if $j\in\Pset$ and negative if $j\in\Dset$.
Each robot has a maximum load capacity $C_i\geq0$ of goods that can be simultaneously held.
The travel time needed for the $i$-th vehicle to move from a location $j \in R$ to another
location $k \in R$ is denoted by $t_i^{jk}\geq0$.
In order to travel from two locations $j,k$, the $i$-th robot incurs a cost $c_i^{jk}\in\real_{\geq0}$.
Finally, two additional locations $s$ and $\sigma$ are considered. The first one
represents
the mission starting point, while the second one is a virtual ending point.
For this reason, the corresponding
demands $q^s, q^\sigma$ and service times
$d^s, d^\sigma$ are set to $0$.

The goal is to construct minimum cost paths satisfying all the transportation requests.
To this end, a graph of all possible paths through the transportation requests is defined as follows.
Let $\GG_A=(\Vset,\EE_A)$, be the graph with vertex set $\Vset=\{s,\sigma\}\cup R$
and edge set $\EE_A=\{(j,k)\mid j,k\in \Vset, j\neq k \text{ and } j\neq \sigma, k\neq s \}$.
Owing to its definition, $\EE_A$ contains edges starting from $s$ or from locations in $R$ and
ending in $\sigma$ or other locations in $R$.
For all edges $(j,k)\in\EE_A$, let $x_i^{jk}$ be a binary variable denoting whether
vehicle $i\in\until{N}$ is traveling ($x_i^{jk}=1$) or not ($x_i^{jk}=0$) from a location $j$
to a location $k$.
Also, let $B_i^j\in\real_{\geq 0}$ be an the optimization variable modeling the time at which
vehicle $i$ begins its service at location $j$. Similarly, let $Q_i^j\in\real_{\geq 0}$ be the
load of vehicle $i$ when leaving location $j$.
To keep the notation light, we denote by $x$ the vector stacking $x_i^{jk}$ for all $i, j, k$
and by $B, Q$ the vectors stacking all $B_i^j$ and $Q_i^j$.
The \ac{PDVRP} can be formulated as the following optimization problem~\cite{parragh2008survey},
\begin{subequations}%
\label{eq:PDVRP}%
\begin{align}
  \min_{
    x, B, Q
  } \: & \: \sum_{i=1}^N \sum_{(j,k) \in \EE_A} c_i^{jk} x_i^{jk} \label{eq:PDVRP_cost}
  \\
  \subj \:
  & \: \sum_{i=1}^N \sum_{k:(j,k) \in \EE_A} x_i^{jk} \geq 1
  \hspace{1.5cm} \forall j \in \Rset
  \label{eq:PDVRP_coupling_constr}
  \\
  & \: \sum_{k:(s,k) \in \EE_A} \!\!\! x_i^{sk} \!= 1
  \hspace{2.32cm} \forall i \in \agents
  \label{eq:PDVRP_local_con1}
  \\
  & \: \sum_{j:(j,\sigma) \in \EE_A} \!\!\! x_i^{j\sigma} \!= 1
  \hspace{2.32cm} \forall i \in \agents
  \label{eq:PDVRP_local_con2}
  \\
  & \: \sum_{j:(j,k) \in \EE_A} \!\!\! x_i^{jk} = \! \sum_{j:(k,j) \in \EE_A} \!\!\!\! x_i^{kj}
  \hspace{0.853cm} \forall i \in \agents, k \in \Rset
  \label{eq:PDVRP_local_con3}
  \\
  & \: \sum_{k:(j,k)\in \EE_A} \!\!\! x_{i}^{jk} = \!\!\!\!\!\!\! \sum_{k:(j+|P|,k)\in \EE_A} \!\!\!\!\!\!\!\!\! x_i^{|P|+j,k}
  \hspace{0.195cm} \forall i\in\agents, j\in P
  \label{eq:PDVRP_local_con4}\\
  & \: B_i^j\leq B_i^{j+|P|}, \hspace{2.6cm} \forall i\in\agents, j\in P
  \label{eq:PDVRP_local_con5}\\
  & \: x_i^{jk} = 1 \Rightarrow B_i^k \geq B_i^j+d^j+t_i^{jk}
  \label{eq:PDVRP_local_nonlcon_1}
  \\
  & \: x_i^{jk} = 1 \Rightarrow Q_i^k = Q_i^j+q^k
  \label{eq:PDVRP_local_nonlcon_2}
  \\
  & \: \underline{Q}^j \leq Q_i^j\leq \overline{Q}^j_i
	\hspace{2.4cm} \forall j\in \Vset, i\in \agents
	\label{eq:PDVRP_local_con6}
	\\
  & \: Q_i^s=Q_i^\text{init} \hspace{3.14cm} \forall i \in \agents
  \label{eq:PDVRP_local_con7}\\
  & \: x_i^{jk} \in \binset
  \hspace{2.1cm} \forall i \in \agents, (j,k) \in \EE_A,
  \label{eq:PDVRP_local_con8}
\end{align}
\end{subequations}
where $\underline{Q}^j=\max\{0,q^j\}$, $\overline{Q}^j_i=\min\{C_i,C_i+q^j\}$
and $Q_i^\text{init}\in\real_{\geq0}$.
We make the standing assumption that problem~\eqref{eq:PDVRP} is feasible
and admits an optimal solution.
Throughout the document, we use the convention that subscripts
denote the vehicle index, while superscripts refer to locations.
Table~\ref{tb:symbols} collects all the relevant symbols.
\begin{table}[t]\centering
	\caption{List of the main symbols and their definitions}
	\label{tb:symbols}
	\begin{tabular}{ll}
	  \hline
		\multicolumn{2}{c}{Basic definitions}
		\\
		\hline
		$\Nag\in\mathbb{N}_{\geq0}$ 
		& Number of vehicles of the system
		\\
		$\agents = \until{N}$
		& Set of vehicles
		\\
		$\Pset$, $\Dset$
		& Sets of Pickup and Delivery requests
		\\
		$R = \Pset \cup \Dset$
		& Set of all transportation requests
		\\
		$s$
		& Mission starting point
		\\
		$\sigma$
		& Mission ending point (virtual or physical)
		\\
		$\Vset = \{s, \sigma \} \cup R$
		& Set of \ac{PDVRP} graph vertices
		\\
		$\EE_A \subset \Vset \times \Vset$
		& Set of \ac{PDVRP} graph edges
		\\
    \hline
		\multicolumn{2}{c}{Optimization variables}
		\\
		\hline
		$x_i^{jk}\in\binset$
		& $1$ if vehicle $i$ travels arc $(j,k)$, $0$ otherwise
		\\
		$Q_i^j\in \mathbb{R}_{\geq0}$
		& Load of vehicle $i$ when leaving vertex $j$
		\\ 
		$B_i^j\in \mathbb{R}_{\geq0}$
		& Beginning of service of vehicle $i$ at vertex $j$
		\\
    \hline
		\multicolumn{2}{c}{Problem data}
		\\
		\hline
		$c_i^{jk} \in \mathbb{R}_{\geq0}$
		& Incurred cost if vehicle $i$ travels arc $(j,k)$
		\\
    $q^j\in\real$
    & Demand/supply at location $j\in \Rset$
    \\
		$C_i\in\real_{\geq0}$
		& Capacity of vehicle $i$
		\\
		$t_i^{jk} \in \mathbb{R}_{\geq0}$
		& Travel time from $j$ to $k$ for vehicle $i$
		\\
		$d^j\in\real_{\geq0}$
		& Service duration at $j\in \Vset$
		\\
		$Q_i^\text{init} \ge 0$
		& Initial load of vehicle $i$
		\\
		\hline
	\end{tabular}
\end{table}
Notice that, in order to satisfy~\eqref{eq:PDVRP_local_con6}, it is necessary to assume
$C_i\geq\max_{j \in R}\{q^j\}$, i.e., the generic task can be performed by at least one robot.
In order to keep the discussion not too technical, in the following we always maintain this assumption.
However, note that the algorithm proposed in this paper works also if this assumption is removed.
A discussion on this extension is given in Section~\ref{sec:discussion_parameters}, from
which it follows that trivial solutions where a robot is assigned all the tasks are not feasible.

We conclude by noting that problem~\eqref{eq:PDVRP} is mixed-integer but not linear.
Indeed, the constraints~\eqref{eq:PDVRP_local_nonlcon_1}--\eqref{eq:PDVRP_local_nonlcon_2} are
nonlinear. However, an equivalent linear formulation of these constraints is always possible
(the detailed procedure is outlined in Appendix~\ref{sec:conv_lin}). In the rest of the paper,
we refer to problem~\eqref{eq:PDVRP} as being a MILP, with the implicit assumption that
the constraints~\eqref{eq:PDVRP_local_nonlcon_1}--\eqref{eq:PDVRP_local_nonlcon_2} are
replaced by their linear version.
Notice that it is not necessary to perform such reformulation when implementing the algorithm. Indeed, several modern solvers allow for the implementation of these constraints by means of so-called \emph{indicator constraints}. We however prefer to consider a linear reformulation in order to streamline the analysis.

\subsection{Description of Cost and Constraints}
\label{sec:cost_constraint_description}

Let us detail the cost and constraints of problem~\eqref{eq:PDVRP}.
The objective~\eqref{eq:PDVRP_cost} minimizes the total route cost, in particular,
the total euclidean distance traveled by robots.

Constraint~\eqref{eq:PDVRP_coupling_constr} enforces that every location has to be
visited at least once. Typically, PDVRP formulations consider this constraint as an
equality, however, in the considered case of cost being the total distance,
the solution is the same both with the equality and with the inequality.
We stick to the inequality formulation as this
allows us to exploit the problem structure and efficiently solve the problem.
Constraints~\eqref{eq:PDVRP_local_con1}--\eqref{eq:PDVRP_local_con2} guarantee that
every vehicle starts its mission at $s$ and ends at $\sigma$.
Equality~\eqref{eq:PDVRP_local_con3} is a flow conservation constraint, meaning that if a
vehicle enters a location $k$ it also has to leave it.
Constraint~\eqref{eq:PDVRP_local_con4} ensures that, if a robot $i$ performs a pickup
operation, it also has to perform the corresponding delivery.
Inequality~\eqref{eq:PDVRP_local_con5} imposes that deliveries have
to occur after pickups.
Constraint~\eqref{eq:PDVRP_local_nonlcon_1} avoids subtours in each vehicle route (i.e. paths passing
from the same location more than once), while inequalities~\eqref{eq:PDVRP_local_nonlcon_2}--\eqref{eq:PDVRP_local_con6}
ensure that the total vehicle capacity is never exceeded.
Finally,~\eqref{eq:PDVRP_local_con7} takes into account the initial load of the
$i$-th robot.

\section{Distributed Algorithm}
\label{sec:distributed_algorithm}
In this section, we propose our distributed algorithm to solve the Multi-vehicle
\ac{PDVRP}. We first present the distributed problem setup. Then, we formally
describe the proposed distributed algorithm.

\subsection{Toward a Distributed Resource Allocation Scheme}
\label{sec:distributed_allocation_scheme}

We assume robots aim to solve problem~\eqref{eq:PDVRP} in a distributed fashion, i.e.
without a central (coordinating) node. In order to solve the problem, we suppose
each robot is equipped with its own communication and computation capabilities.
Robots can exchange information according to a static communication network
modeled as a connected and undirected graph $\GG=(\until{N},\EE)$.
The graph $\GG$ models the communication in the
sense that there is an edge $(i,\neigh) \in \EE$ if and only if robot $i$ is able
to send information to robot $\neigh$.
For each node $i$, the set of neighbors of $i$ is denoted 
by $\innbrs_{i} = \{ \neigh \in \agents : (i,\neigh) \in \EE \}$.
We consider the challenging scenario in which the $i$-th robot only
knows problem data related to it, namely the $i$-th travelling times $t_i^{jk}$, the $i$-th local capacity $C_i$
and the $i$-th cost entries $c_i^{jk}$, thus not having access to the entire problem formulation.
However, we reasonably assume that all the robots know the demand/supply
values $q_j$ and service time $d_j$ for each task request $j\in\Vset$.

Note that the optimization variables in~\eqref{eq:PDVRP} associated with a robot $i$
are all and only the variables with subscript $i$ (i.e. $x_i^{jk}$, $B_i^j$ and $Q_i^j$ for all $j, k$).
In order to simplify the notation,
let us define for all $i$ the set
\begin{align}
  Z_i = \Big\{
    (x_i, B_i, Q_i) \text{ such that
    \eqref{eq:PDVRP_local_con1}--\eqref{eq:PDVRP_local_con8} are satisfied}
  \Big\}.
\end{align}
Indeed, note that the constraints~\eqref{eq:PDVRP_local_con1}--\eqref{eq:PDVRP_local_con8}
are repeated for each index $i$.
With this shorthand, any route that robot $i$ can implement can be denoted
more shortly as $(x_i, B_i, Q_i) \in Z_i$.
However, note that feasible solutions to problem~\eqref{eq:PDVRP} must not only be valid
routes, but they must also satisfy the pickup/delivery demand.
More formally, the vectors $(x_i, B_i, Q_i)$ must satisfy both the constraints
\begin{subequations}
\begin{align}
  & (x_i, B_i, Q_i) \in Z_i \hspace{1.5cm} \forall i \in \agents,
  \\
   \text{ and }\quad  &\sum_{i=1}^N \sum_{k:(j,k) \in \EE_A} x_i^{jk} \geq 1
  \hspace{0.75cm} \forall j \in \Rset,
  \label{eq:coupling_constr}
\end{align}%
\end{subequations}
where~\eqref{eq:coupling_constr} is the coupling constraint~\eqref{eq:PDVRP_coupling_constr}.
In light of this observation, the main idea to devise a distributed algorithm for
problem~\eqref{eq:PDVRP} is to perform a negotiation to determine the value of
the left-hand side of the constraint~\eqref{eq:coupling_constr} in a distributed way.
This technique is known as \emph{primal decomposition}
(see also Appendix~\ref{sec:primal_decomp}).
Let us define \emph{allocation} vectors $y_i \in \real^{|\Rset|}$ that add up to the
right-hand side of~\eqref{eq:coupling_constr}, i.e.
\begin{align*}
  \sum_{i=1}^N [y_i]_j = 1, \hspace{0.75cm} \forall j \in \Rset,
\end{align*}
where $[y_i]_j$ denotes the $j$-th component of $y_i$.
The variables $y_i$ should be interpreted as the allocation of a resource which is shared among the robots (cf. Appendix~\ref{sec:primal_decomp}).
Each robot $i$ will aim to determine its allocation $y_i$ in such a way that
\begin{align*}
  \sum_{k:(j,k) \in \EE_A} x_i^{jk} \ge [y_i]_j,
  \hspace{0.75cm} \forall j \in \Rset,
\end{align*}
from which it directly follows that~\eqref{eq:coupling_constr} is satisfied.
As it will be clear from the
forthcoming analysis, the $j$-th entry of the vector $y_i$ determines
whether or not robot $i$ must perform task $j$.
In the next subsection, we introduce our distributed algorithm, whose purpose is
to coordinate the computation of allocation vectors $y_i$. such
that~\eqref{eq:coupling_constr} is satisfied.

\subsection{Distributed Algorithm Description}
\label{sec:alg_description}
Let us now introduce our distributed algorithm. Let $t \in \natural$ denote the iteration index.
Each robot $i$ maintains an estimate of the local allocation vector $y_i^t \in \real^{|R|}$.
At each iteration, the vector $y_i^t$ is updated according
to~\eqref{eq:alg_distributed_z_LP}--\eqref{eq:alg_distributed_y_update}.
After a finite number of iterations, say $T_f \in \natural$, the robots compute a tentative solution to the \ac{PDVRP} based
on the last computed allocation $y_i^{T_f}$ with~\eqref{eq:alg_rounding}--\eqref{eq:alg_MILP}.
The whole algorithm can be seen as a subgradient method applied to a suitable, convex reformulation of problem~\eqref{eq:PDVRP}.
Algorithm~\ref{alg:algorithm} summarizes the scheme as performed by each robot $i$.
The symbol $\conv{Z_i}$ denotes the convex hull of the set $Z_i$,
while $\alpha^t$ is a step-size sequence. The algorithm has also some tuning
parameters that are reported on the top of the table.

\begin{algorithm}[htbp]
\floatname{algorithm}{Algorithm}

  \begin{algorithmic}[0]
    \Statex \textbf{Parameters}:
       $T_f > 0$, $M > 0$, $0 < \delta < 1$.
       \smallskip
    
    \Statex \textbf{State}: $y_i \in \real^{|\Rset|}$ (initialized at $[y_{i}^0]_j = \delta/N$ for all $j \in \Rset$).
    \medskip

    \Statex \textbf{Repeat} for $t = 0, 1, \ldots, T_f-1$:
    \smallskip
    \begin{tikzpicture}[remember picture, overlay]
      \draw (-4.7,-0.28) -- +(0,-4.9);
    \end{tikzpicture}
    
      \StatexIndent[0.75]
      \textbf{Compute} $\mu_i^t$ as Lagrange multiplier of linear program
      \begin{align} 
      \label{eq:alg_distributed_z_LP}
      \begin{split}      
        \min_{x_i, B_i, Q_i, v_i} \: &\: \sum_{(j,k) \in \EE_A} c_i^{jk} x_i^{jk} + M v_i
        \\
        \subj \:
        & \: \sum_{k:(j,k) \in \EE_A} x_i^{jk} \geq [y_i^t]_j - v_i, \qquad \forall j \in \Rset
        \\
        & \: (x_i, B_i, Q_i) \in \conv{Z_i}, \:\: v_i \ge 0
      \end{split}
      \end{align}
      \smallskip

      \StatexIndent[0.75]
      \textbf{Receive} $\mu_{\neigh}^t$ from neighbors $\neigh\in\nbrs_i$ and update
      \begin{align}
      \label{eq:alg_distributed_y_update}
        y_{i}^{t+1} = y_{i}^t - \alpha^t \sum_{\neigh \in \nbrs_i} \big( \mu_{i}^t - \mu_{\neigh}^t \big)
      \end{align}

      \Statex
      \textbf{Perform} component-wise thresholding of allocation
      \begin{align}
        [\yend_i]_j = \min\Big( [y_i^{T_f}]_j, 1 \Big),\hspace{1cm} \forall j \in \Rset
      \label{eq:alg_rounding}
      \end{align}
      
      \Statex
      \textbf{Return} $(\xend_i, \Bend_i, \Qend_i)$ as optimal solution of MILP
      \begin{align}
      \label{eq:alg_MILP}
      \begin{split}      
        \min_{x_i, B_i, Q_i} \: & \: \sum_{(j,k) \in \EE_A} c_i^{jk} x_i^{jk}
        \\
        \subj \: 
        & \: \sum_{k:(j,k) \in \EE_A} x_i^{jk} \geq [\yend_i]_j, \qquad \forall j \in \Rset
        \\
        & \: (x_i, B_i, Q_i) \in Z_i
      \end{split}
      \end{align}
      
  \end{algorithmic}
  \caption{\algname/}
  \label{alg:algorithm}
\end{algorithm}

Let us informally comment on the algorithm table. The algorithm is composed of two
logic blocks. The first block, represented by the steps~\eqref{eq:alg_distributed_z_LP}--\eqref{eq:alg_distributed_y_update},
is repeated in an iterative manner and is aimed at computing a final allocation vector $y_i^{T_f}$.
Notice that in~\eqref{eq:alg_distributed_z_LP} we replace the mixed-integer set $Z_i$ with a convex, polyhedral set. This makes the problem easier to solve. Moreover, thanks to the Shapley-Folkman lemma, part of the optimal decision variables for~\eqref{eq:alg_distributed_z_LP} satisfy the binary constraints. The integrality of the remaining optimization variables is recovered at the end of the algorithm in~\eqref{eq:alg_MILP}. 
This significantly reduces the computational burden of the scheme.
As for the update in~\eqref{eq:alg_distributed_y_update}, it is a subgradient-based iteration on the resource allocation variable $y_i$. An analysis of this scheme is provided in Section~\ref{sec:analysis_dpd}.
In the second block, represented by the steps~\eqref{eq:alg_rounding}--\eqref{eq:alg_MILP},
the final allocation is used to determine a solution to the original problem.
The thresholding step~\eqref{eq:alg_rounding}
rectifies the current local allocation $y_i^{T_f}$ to obtain the final allocation $\yend_i$,
which is fed to problem~\eqref{eq:alg_MILP}
to compute the local portion of solution $(\xend_i, \Bend_i, \Qend_i)$.
An analysis of the algorithm is provided in Section~\ref{sec:analysis}, while
a discussion on the choice of the tunable parameters $M$, $T_f$ and $\delta$
can be found in Section~\ref{sec:discussion_parameters}.

A few additional remarks are in order. First, note that the only information
exchanged with neighbors are the Lagrange
multipliers $\mu_i^t$. Thus, robots never exchange sensitive local information
(such as the cost, the load or the computed routes),
therefore the algorithm maintains privacy.
Remarkably, even though the number of robots increase, the amount of local
computation remains unchanged. This scalability property is attractive and
enables the solution of the problem even in large networks of robots.

In order to state the theoretical properties of the algorithm, we make the following
standard assumption on the step-size $\alpha^t$
appearing in~\eqref{eq:alg_distributed_y_update}.
\begin{assumption}
\label{ass:step-size}
  The step-size sequence $\{ \alpha^t \}_{t\ge0}$, with each $\alpha^t \ge 0$,
  satisfies $\sum_{t=0}^{\infty} \alpha^t = \infty$,
  $\sum_{t=0}^{\infty} \big( \alpha^t \big)^2 < \infty$.
  \oprocend
\end{assumption}
A discussion on possible choices of the step-size satisfying Assumption~\ref{ass:step-size}
is given in Section~\ref{sec:discussion_parameters}.
The next theorem represents the central result of the paper.
\begin{theorem}[Finite-time feasibility]
\label{thm:finite_time_feasibility}
  Let Assumption~\ref{ass:step-size} hold and let $0 < \delta < 1$.
  Then, for a sufficiently large $M > 0$, there exists a time $T_\delta > 0$
  such that the vector $(\zend_1, \ldots, \zend_N)$, the aggregate output
  of Algorithm~\ref{alg:algorithm}, with each $\zend_i = (\xend_i, \Bend_i, \Qend_i)$,
  is a feasible solution to problem~\eqref{eq:PDVRP}, provided that the total
  iteration count satisfies $T_f \ge T_\delta$.
\oprocend
\end{theorem}
The proof is provided in Section~\ref{sec:proof_theorem}.
In Section~\ref{sec:simulations}, we perform an empirical study on the
optimality of the solutions computed by Algorithm~\ref{alg:algorithm}.

\begin{remark}[Algorithm complexity]
  Note that the steps~\eqref{eq:alg_rounding}--\eqref{eq:alg_MILP} are performed only once at the end of the algorithm.  Steps~\eqref{eq:alg_distributed_z_LP}--\eqref{eq:alg_distributed_y_update} instead are performed $T_f$ times.
   Let $L_i$ denote the number of constraints of $\conv{Z_i}$. Then, the average complexity of~\eqref{eq:alg_distributed_z_LP} (solved via the simplex algorithm) is $\OO(L_i + |R|)$.
   Step~\eqref{eq:alg_distributed_y_update} consists of sums and multiplications and its complexity is $\OO(|\mathcal{N}_i||R|)$.  
   Step~\eqref{eq:alg_rounding} is a thresholding with complexity $\OO(|R|)$.
    Excluding the MILP~\eqref{eq:alg_MILP}, the total algorithm complexity is thus $\OO(T_f(|\mathcal{N}_i||R| + L_i))$.
     Finally,~\eqref{eq:alg_MILP} can be solved using a branch-and-bound scheme, which complexity is exponential in the number $O(|R|^2)$ of integer variables.
\oprocend
\end{remark}

\section{Algorithm Analysis}
\label{sec:analysis}

In this section, we provide a theoretical study of Algorithm~\ref{alg:algorithm}.
The algorithm inherits the ideas of~\cite{camisa2019milp}, however here
we are considering an enlargement of the constraints rather than a restriction,
and moreover there is a thresholding operation which is specific of
Algorithm~\ref{alg:algorithm}. We provide a
compact analysis focused on those properties that are peculiar to the PDVRP
and that are not considered elsewhere.
To arrive at the final result given by Theorem~\ref{thm:finite_time_feasibility},
we will proceed with the following steps.
\begin{enumerate}
	\item \label{en:feas} We show that the steps in~\eqref{eq:alg_distributed_z_LP} and~\eqref{eq:alg_MILP} are well posed (Section~\ref{sec:feasibility}).
	\item \label{en:resource_alloc} We analyze the steps in~\eqref{eq:alg_distributed_z_LP} and~\eqref{eq:alg_distributed_y_update}, needed to retrieve an optimal solution to a relaxed PDVRP (Section~\ref{sec:analysis_dpd}).
	\item Finally, we prove Theorem~\ref{thm:finite_time_feasibility} (Section~\ref{sec:proof_theorem}) with the help of auxiliary technical lemmas (Section~\ref{sec:intermediate_results}).
\end{enumerate}
From now on, we denote by $\1$ the vector of ones with appropriate dimension.

\subsection{Feasibility of Local Problems}
\label{sec:feasibility}
We begin the analysis by proving that the algorithm is well posed. In particular, we show that it is indeed
possible to solve the problems~\eqref{eq:alg_distributed_z_LP} and~\eqref{eq:alg_MILP}.
The next lemma formalizes feasibility of problem~\eqref{eq:alg_distributed_z_LP}.
\begin{lemma}
\label{lemma:well_posed_conv}
  Consider a robot $i \in \agents$ and let $y_i^t$ be allocation computed by
  Algorithm~\ref{alg:algorithm} at an iteration $t$.
  Then, problem~\eqref{eq:alg_distributed_z_LP} is feasible.
\end{lemma}
\begin{proof}
	See Appendix~\ref{sec:proof_lemma_well_posed_conv}.
\end{proof}

Proving feasibility of problem~\eqref{eq:alg_MILP} is more delicate and
relies upon the thresholding operation, as formally shown next.
\begin{lemma}
\label{lemma:well_posed_milp}
  Consider a robot $i \in \agents$ and let $\yend_i$ be the final allocation computed by Algorithm~\ref{alg:algorithm}.
  Then, problem~\eqref{eq:alg_MILP} is feasible.
\end{lemma}
\begin{proof}
	See Appendix~\ref{sec:proof_lemma_well_posed_milp}.
\end{proof}

\subsection{Convergence of Distributed Resource Allocation Scheme}
\label{sec:analysis_dpd}

We now focus on the first logic block of Algorithm~\ref{alg:algorithm}, namely
steps~\eqref{eq:alg_distributed_z_LP}--\eqref{eq:alg_distributed_y_update}.
These two iterative steps can be used to
obtain an optimal allocation associated with the linear program
\begin{align}
\begin{split}
  \min_{x, B, Q} \: & \: \sum_{i=1}^N \sum_{(j,k) \in \EE_A} c_i^{jk} x_i^{jk}
  \\
  \subj \:
  & \: \sum_{i=1}^N \sum_{k:(j,k) \in \EE_A} x_i^{jk} \geq \delta
  \hspace{1.4cm} \forall j \in \Rset
  \\
  & \: (x_i, B_i, Q_i) \in \conv{Z_i}, \hspace{1cm} \forall \: i \in \agents.
\end{split}
\label{eq:convexified_pb}
\end{align}
Formally, we denote with $(\cy_1, \ldots, \cy_N)$ such optimal allocation, which is the optimal vector satisfying the conditions specified in Section~\ref{sec:distributed_allocation_scheme}.
Before stating formally this result, we note that, by the discussion in
Section~\ref{sec:distributed_allocation_scheme}, problem~\eqref{eq:convexified_pb}
is essentially the same as
problem~\eqref{eq:PDVRP}, except that the mixed-integer constraints
$(x_i, B_i, Q_i) \in Z_i$ are replaced by their convex relaxation
$(x_i, B_i, Q_i) \in \conv{Z_i}$. Moreover,
the coupling constraints are enlarged by changing the right-hand side from $1$ to $\delta \in (0,1)$
(recall that $\delta$ is a
tunable parameter of Algorithm~\ref{alg:algorithm}).

Now, denote by $(\cz_1, \ldots, \cz_N)$ an optimal
solution of problem~\eqref{eq:convexified_pb} with each $z_i = (x_i, B_i, Q_i)$,
and by $\{y_1^t, \ldots, y_N^t\}_{t \ge 0}$ the allocation vector sequence produced
by~\eqref{eq:alg_distributed_z_LP}--\eqref{eq:alg_distributed_y_update}.
The following lemma %
summarizes the convergence properties of the distributed resource allocation
scheme~\eqref{eq:alg_distributed_z_LP}--\eqref{eq:alg_distributed_y_update}.
\begin{lemma}
\label{lemma:DPD_convergence}
	Let Assumption~\ref{ass:step-size} hold and recall that $y_i^0 = \delta/N \1$
	for all $i \in \agents$.
	Moreover, let $(\cy_1, \ldots, \cy_N) \in \real^{N|\Rset|}$ be an optimal allocation associated with
	problem~\eqref{eq:convexified_pb}, i.e., a vector satisfying $\sum_{i=1}^N \cy_i = \delta \1$
	and $\sum_{k:(j,k) \in \EE_A} x_i^{jk} \geq [\cy_i]_j$ for all $j \in \Rset$ and $i \in \agents$.
	Then, for a sufficiently large $M > 0$, the distributed
	algorithm~\eqref{eq:alg_distributed_z_LP}--\eqref{eq:alg_distributed_y_update}
	generates a sequence
	$\{ y_1^t,\ldots,y_N^t \}_{t\ge 0}$
	such that
  \begin{enumerate}
	  \item $\sum_{i=1}^N y_i^t  = \delta \1$, for all $t \ge 0$;
	  
	  \item $\lim_{t \to \infty} \| y_i^t - \cy_i \| = 0$ for all
	    $i \in \agents$.
  \end{enumerate}
\end{lemma}
\begin{proof}
We refer the reader to~\cite{camisa2019milp} for the proof.
\end{proof}

\subsection{Intermediate Results}
\label{sec:intermediate_results}

Before turning to the proof of Theorem~\ref{thm:finite_time_feasibility},
we provide some preparatory lemmas.
The first lemma justifies the use of $\delta$ (an arbitrarily small
positive number) in place of the original right-hand side
$1$ in the coupling constraints~\eqref{eq:PDVRP_coupling_constr}.
\begin{lemma}
\label{lemma:rhs_relaxation}
  For all $i \in \agents$, let $(x_i, B_i, Q_i) \in Z_i$ such that
  $\sum_{i=1}^N \sum_{k:(j,k) \in \EE_A} x_i^{jk} > 0$ for all $j \in \Rset$.
  Then, it holds $\sum_{i=1}^N \sum_{k:(j,k) \in \EE_A} x_i^{jk} \ge 1$ for all $j \in \Rset$.
\end{lemma}
\begin{proof}
  For each component $j \in \Rset$, by assumption we have
  \begin{align*}
    \sum_{i=1}^N \sum_{k:(j,k) \in \EE_A} x_i^{jk}
    > 0.
  \end{align*}
  Note that, since each $x_i^{jk} \in \{0, 1\}$, the quantity
  $\sum_{i=1}^N \sum_{k:(j,k) \in \EE_A} x_i^{jk}$ is either zero
  or at least equal to $1$. Therefore, because of the assumption,
  we have $\sum_{i=1}^N \sum_{k:(j,k) \in \EE_A} x_i^{jk} \ge 1$
  for all $j \in \Rset$.
\end{proof}

The next two lemmas will be used in the sequel
to characterize an optimal allocation associated with problem~\eqref{eq:convexified_pb}.
\begin{lemma}
\label{lemma:zero_restriction}
  For all $i \in \agents$, let $\ty_i \in \real^{|\Rset|}$ and
  $(\tx_i, \tB_i, \tQ_i) \in \conv{Z_i}$ such that
  $\sum_{k:(j,k) \in \EE_A} \tx_i^{jk} > [\ty_i]_j$ for all $j \in \Rset$.
  Then, there exists $(\bx_i, \bB_i, \bQ_i) \in Z_i$ satisfying
  $\sum_{k:(j,k) \in \EE_A} \bx_i^{jk} > [\ty_i]_j$ for all $j \in \Rset$.
\end{lemma}
\begin{proof}
Fix a robot $i \in \agents$ and note that, because of the
  flow constraints~\eqref{eq:PDVRP_local_con3} and the subtour elimination
  constraints~\eqref{eq:PDVRP_local_nonlcon_1}, for all $j \in \Rset$ it holds
  \begin{align}
    \max_{(x_i, B_i, Q_i) \in Z_i} \bigg( \sum_{k:(j,k) \in \EE_A} x_i^{jk} \bigg) = 1.
  \label{eq:max_local_coupling}
  \end{align}
  Moreover, note that it is possible to choose a local solution passing through all the
  locations (possibly at a high cost), i.e., there exists $(\bx_i, \bB_i, \bQ_i) \in Z_i$
  such that $\sum_{k:(j,k) \in \EE_A} \bx_i^{jk} = 1$ for all $j \in \Rset$.
  Therefore, for all $j \in \Rset$ it holds
  \begin{align}
    &\sum_{k:(j,k) \in \EE_A} \bx_i^{jk}
    \nonumber
    \\
    &\hspace{0.2cm}= \max_{(x_i, B_i, Q_i) \in Z_i} \bigg( \sum_{k:(j,k) \in \EE_A} x_i^{jk} \bigg)
    \nonumber
    \\
    &\hspace{0.2cm}\stackrel{(a)}{=}  \max_{(x_i, B_i, Q_i) \in \conv{Z_i}} \: \bigg( \sum_{k:(j,k) \in \EE_A} x_i^{jk} \bigg)
    \nonumber
    \\
    &\hspace{0.2cm}\ge \sum_{k:(j,k) \in \EE_A} x_i^{jk} \hspace{0.5cm} \text{ for all } (x_i, B_i, Q_i) \in \conv{Z_i},
  \end{align}
  where \emph{(a)} follows by linearity of the cost.
  In particular, the previous inequality holds with $(x_i, B_i, Q_i) = (\tx_i, \tB_i, \tQ_i)$ and thus
  for all $j \in \Rset$ we have
  $\sum_{k:(j,k) \in \EE_A} \bx_i^{jk} \ge \sum_{k:(j,k) \in \EE_A} \tx_i^{jk} > [\ty_i]_j$.
\end{proof}

\begin{lemma}
\label{lemma:optimal_allocation_characterization}
  Let $(\cy_1, \ldots, \cy_N) \in \real^{N|\Rset|}$ be an optimal allocation associated with
	problem~\eqref{eq:convexified_pb}, i.e., a vector satisfying $\sum_{i=1}^N \cy_i = \delta \1$
	and $\sum_{k:(j,k) \in \EE_A} x_i^{jk} \geq [\cy_i]_j$ for all $j \in \Rset$ and $i \in \agents$.
  Then, $\cy_i \le \1$ for all $i \in \agents$.
\end{lemma}
\begin{proof}
  By contradiction, suppose that there is a component $j \in \Rset$
  for which $[\cy_i]_j > 1$. By assumption, we have $\sum_{k:(j,k) \in \EE_A} x_i^{jk} \ge [\cy_i]_j$.
  Using Lemma~\ref{lemma:zero_restriction}, we conclude that
  there exists $(\bx, \bB, \bQ) \in Z_i$ such that
  \begin{align*}
    \sum_{k:(j,k) \in \EE_A} \bx_i^{jk}
    \ge
    [\cy_i]_j
    > 1,
  \end{align*}
  which contradicts~\eqref{eq:max_local_coupling}.
\end{proof}

\subsection{Proof of Theorem~\ref{thm:finite_time_feasibility}}
\label{sec:proof_theorem}
  First, note that, by Lemmas~\ref{lemma:well_posed_conv} and~\ref{lemma:well_posed_milp},
  the algorithm is well posed.
  Moreover, by construction (cf. problem~\eqref{eq:alg_MILP})
  it holds $(\xend_i, \Bend_i, \Qend_i) \in Z_i$
  for all $i \in \agents$. %
  Therefore, all the local constraints \eqref{eq:PDVRP_local_con1} to~\eqref{eq:PDVRP_local_con8}
  are satisfied by $(\xend_i, \Bend_i, \Qend_i)$ and we only need to show
  that there exists $T_\delta > 0$ such that constraint~\eqref{eq:PDVRP_coupling_constr}
  is satisfied by $(\xend_i, \Bend_i, \Qend_i)$ %
  if $T_f \ge T_\delta$.
  By Lemma~\ref{lemma:rhs_relaxation}, it suffices to prove
  that $\sum_{i=1}^N \sum_{k:(j,k) \in \EE_A} x_i^{jk} > 0$ for all $j \in \Rset$.
  
  Consider the auxiliary sequence $\{y_1^t, \ldots, y_N^t\}_{t \ge 0}$
  generated by Algorithm~\ref{alg:algorithm}. By Lemma~\ref{lemma:DPD_convergence},
  this sequence converges to the vector $(\cy_1, \ldots, \cy_N)$.
  By definition of limit (using the infinity norm), there exists $T_{\delta} > 0$ such that
  $\Vert y_i^t - \cy_i \Vert_\infty < \delta/N$
  (and thus $y_i^t < \cy_i + \delta/N \1$)
  for all $i \in \agents$ and $t \ge T_\delta$.
  
  Let us define a vector $\rho_i \in \real^{|\Rset|}$ representing the mismatch
  between $y_i^{T_f}$ and its thresholded version $\yend_i$,
  \begin{align}
     \rho_i = y_i^{T_f} - \yend_i,
     \hspace{1cm}
     \text{for all } i \in \agents.
  \end{align}
  By definition~\eqref{eq:alg_rounding}, it holds $\rho_i \ge 0$.
  Then, for all $j \in \Rset$, it holds
  \begin{align*}
    \sum_{i=1}^N \sum_{k:(j,k) \in \EE_A} {\xend_i}^{jk}
    &\ge \sum_{i=1}^N [\yend_i]_j
    \\
    &= \underbrace{\sum_{i=1}^N [y_i^{T_f}]_j}_{\delta} - \sum_{i=1}^N [\rho_i]_j
    \\
    &= \delta - \sum_{i=1}^N [\rho_i]_j
  \end{align*}
  Let us temporarily assume that $\rho_i < \delta/N \1$ for all $i \in \agents$.
  Then, we obtain the desired statement
  \begin{align*}
    \sum_{i=1}^N \sum_{k:(j,k) \in \EE_A} {\xend_i}^{jk}
    &\ge \delta - \sum_{i=1}^N [\rho_i]_j
    \\
    &> \bigg( \delta - \sum_{i=1}^N \delta/N \bigg) = 0.
  \end{align*}
  It remains to show that $\rho_i < \delta/N \1$ for all $i \in \agents$.
  Fix a robot $i$ and consider a component $j \in \Rset$ of the vector $\rho_i$.
  Owing to the definition~\eqref{eq:alg_rounding} of $\yend_i$, either the $j$-th
  component is equal to $[y_i^{T_f}]_j$ or it is equal to $1$. In the former case,
  we have $[\rho_i]_j = 0 < \delta/N$. In the latter case, there is a non-negative
  mismatch $[\rho_i]_j = [y_i^{T_f}]_j - 1 \ge 0$.
  Now, using the fact $y_i^t < \cy_i + \delta/N \1$ for all $t \ge T_\delta$,
  we have
  \begin{align*}
    [\rho_i]_j
    &= [y_i^{T_f}]_j - 1
    \\
    &< [\cy_i]_j - 1 + \delta/N
    \\
    &\le \delta/N
  \end{align*}
  provided that $T_f \ge T_\delta$, where in the last inequality
  we applied Lemma~\ref{lemma:optimal_allocation_characterization}.
  The proof follows.
  \oprocend
\section{Discussion and Extension}

In this section, we provide guidelines for the choice of the algorithm
parameters and we discuss a possible extension of Algorithm~\ref{alg:algorithm}
to a more general setting.

\subsection{On the Choice of the Parameters}
\label{sec:discussion_parameters}
As already mentioned in Section~\ref{sec:distributed_algorithm}, there are a few parameters
that must be appropriately set in order for Algorithm~\ref{alg:algorithm} to work correctly.
The basic requirements for the parameters are summarized in Theorem~\ref{thm:finite_time_feasibility}
and are recalled here: \emph{(i)} $M > 0$ must be sufficiently large, \emph{(ii)} $\delta$ is any number
in the open interval $(0,1)$, \emph{(iii)} the total number of iterations $T_f > 0$ must be sufficiently
large, \emph{(iv)} the step-size sequence $\{\alpha^t\}_{t \ge 0}$ must satisfy Assumption~\ref{ass:step-size}.

The parameter $M > 0$ is an exact penalty weight (cf.~\cite{bertsekas1982constrained}).
The minimum admissible value is the $1$-norm of the dual solution
of problem~\eqref{eq:convexified_pb}. A conservative choice is any large number not creating
numerical instability when solving the local problems~\eqref{eq:alg_distributed_z_LP}.
The assumption on the step-size sequence $\{\alpha^t\}_{t \ge 0}$ is typical in the distributed
optimization literature. A sensible choice satisfying Assumption~\ref{ass:step-size} is
$\alpha^t = K/(t + 1)$, with any $K > 0$.

The purpose of the parameter $\delta$ is to enlarge the constraint~\eqref{eq:PDVRP_coupling_constr}
(see also Lemma~\ref{lemma:rhs_relaxation}) and is linked
to the minimum value of $T_f$, i.e., the minimum number of iterations
to guarantee feasibility (cf. Theorem~\ref{thm:finite_time_feasibility}).
There is an inherent tradeoff between $\delta$
and the minimal $T_f$. For $\delta$ close to $1$, precedence is given to feasibility
and $T_f$ may become smaller, while for $\delta$ close to $0$, solution optimality
is prioritized, possibly at the cost of a higher $T_f$.
Indeed, for $\delta$ close to $1$, the time
$T_\delta$ in the proof of Theorem~\ref{thm:finite_time_feasibility} may be
smaller, and therefore a smaller number of iterations $T_f$ may be sufficient to
attain feasibility of the solution. Instead, for $\delta$ close to $0$,
a greater number of iterations may be required to obtain feasibility.
However, in the latter case there could be less robots for which the
components of $y_i^{T_f}$ are positive (because, by Lemma~\ref{lemma:DPD_convergence},
it must hold $\sum_{i=1}^N y_i^{T_f} = \delta\1$ with a small positive $\delta$),
thus forcing less robots to pass through the same location.
In Section~\ref{sec:simulations}, we perform simulations in order to study
the trade-off between $\delta$ and $T_f$ numerically.

\subsection{Extension to Heterogeneous PDVRP Graphs}

Let us outline a possible extension of problem~\eqref{eq:PDVRP} that can be
handled by Algorithm~\ref{alg:algorithm}.
Recall from Section~\ref{sec:cost_constraint_description} that problem~\eqref{eq:PDVRP}
has the implicit assumption that $C_i\geq\max_{j \in R}\{q^j\}$ for all $i \in \agents$,
where $C_i$ is the capacity of vehicle $i$ and $q^j$ is the demand/supply at location
$j$. In real scenarios, while there may be vehicles potentially capable of performing all
the pickup/delivery requests, it is often the case that many vehicles are small sized
and can only accomplish a subset of the task requests. This means that the assumption
$C_i\geq\max_{j \in R}\{q^j\}$ may not hold for some robots.

The general case just outlined can be handled with minor modifications in the
formulation of problem~\eqref{eq:PDVRP} and in the algorithm. Indeed, if
$C_i < q^j$ for some robot $i$ and some location $j$, the PDVRP~\eqref{eq:PDVRP}
would be unsolvable by construction (due to infeasibility).
Thus, for each robot $i \in \agents$ we define the largest \emph{local}
set of requests $\Rset_i \subseteq \Rset$ such that
$C_i \geq q^j$ for all $j \in \Rset_i$. Following the description in
Section~\ref{sec:optimization_problem}, the sets $\Rset_i$ will now induce \emph{local}
graphs ${\GG_A}_i=({\Vset}_i,{\EE_A}_i)$ of possible paths,
with vertex set ${\Vset}_i=\{s,\sigma\}\cup \Rset_i$
and edge set ${\EE_A}_i=\{(j,k)\mid j,k\in {\Vset}_i, j\neq k \text{ and } j\neq \sigma, k\neq s \}$.
Then, each robot $i$ defines a smaller set of optimization variables $x_i^{jk}$ with $j,k \in {\Vset}_i$
(instead of $j,k \in \Vset$) and similarly for $B_i$ and $Q_i$.
The optimization problem is formulated similarly to problem~\eqref{eq:PDVRP}
by dropping all the references to non-existing optimization
variables. The resulting modified version of problem~\eqref{eq:PDVRP} is now feasible
as long as for all $j \in \Rset$ there exists $i \in \agents$ such that $C_i \geq q^j$
(i.e., each task can be performed by at least one robot).

The distributed algorithm also requires minor modifications. In particular, the summations
in problems~\eqref{eq:alg_distributed_z_LP} and~\eqref{eq:alg_MILP} are performed using
${\EE_A}_i$ in place of $\EE_A$. Moreover, the thresholding operation~\eqref{eq:alg_rounding}
is replaced by the following one
\begin{align*}
  [\yend_i]_j =
  \begin{cases}
    \min\big( [y_i^{T_f}]_j, 1 \big) & \text{if } j \in \Rset_i
    \\
    \min\big( [y_i^{T_f}]_j, 0 \big) & \text{otherwise}
  \end{cases}
\end{align*}
for all $j \in \Rset$. Finally, the sets $Z_i$ must be replaced by the new version
of the constraints~\eqref{eq:PDVRP_local_con1}--\eqref{eq:PDVRP_local_con8}
with ${\EE_A}_i$ in place of $\EE_A$.
It is possible to follow essentially the same line of proof outlined in
Section~\ref{sec:analysis} with only minor precautions in
Lemmas~\ref{lemma:well_posed_milp} and~\ref{lemma:zero_restriction},
by which it can be concluded that Theorem~\ref{thm:finite_time_feasibility}
holds with no changes.

\section{Simulations on Gazebo}
\label{sec:simulations}
In this section, we provide simulation results for the proposed distributed algorithm
for teams of TurtleBot3 Burger ground robots that have to serve a set of pickup and
delivery requests scattered in the environment. All the simulations are performed using the
\textsc{ChoiRbot}~\cite{testa2020choirbot} ROS~2 framework. We first describe
how we integrate the proposed distributed scheme in \textsc{ChoiRbot}
and then we show simulation results.

\subsection{Simulation Set-up}
In the \textsc{ChoiRbot} software architecture,
each robot is modeled as a cyber-physical agent and consists
of three interacting layers:
distributed optimization layer, trajectory planning layer and low-level control layer.
The \textsc{ChoiRbot} architecture is based on the novel ROS~2 framework,
which handles inter-process communications via the TCP/IP stack.
This allows us to implement the proposed distributed scheme on a real WiFi  network
in which robots communicate with few neighbors according to a given graph.
Robots are simulated in the Gazebo environment~\cite{koenig2004design},
which provides an accurate estimation of the TurtleBot dynamics.
The resulting simulations
are so accurate that the experiments performed in the next section are obtained without any
change or tuning in the code. In this sense, the proposed results are
comparable to experimental results on a real team of robots.
In Figure~\ref{fig:gazebo}, we show a snapshot of one of the simulations
addressed in the next section.
Due to packet losses in the ROS~2 communication middleware with a large number of nodes on a single machine, we were not able to run simulations with more than $30$ robots.
\begin{figure}[htbp]
\centering
  \includegraphics[width=.95\columnwidth]{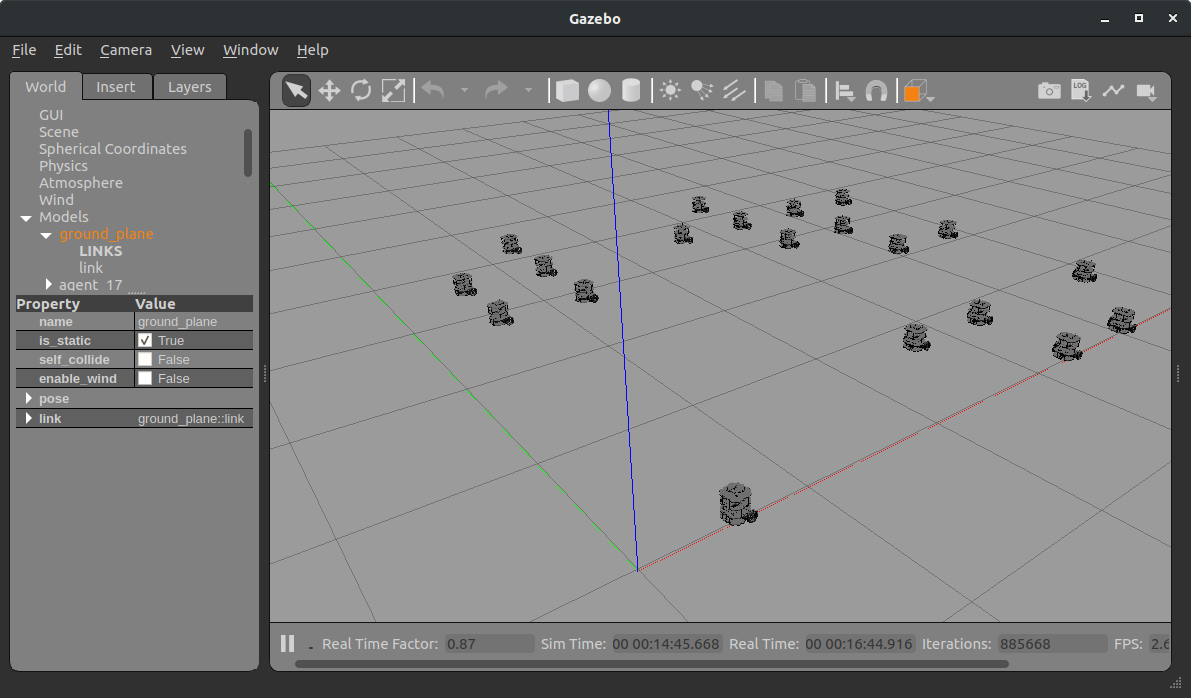}
  \caption{Snapshot of the initial condition of one of the Gazebo simulations.}
  \label{fig:gazebo}
\end{figure}

We now describe more in detail the components of the proposed architecture. 
As said, the distributed optimization layer handles the cooperative solution of the pickup
and delivery problem. It consists of a set of optimization processes,
one for each robot, that perform the steps of
Algorithm~\ref{alg:algorithm}.
At the beginning of the simulation, each optimization node
gathers the information on the pickup and delivery
requests and evaluates the cost vector $c_i$ (i.e. the robot-to-task distances)
and the local constraint sets $Z_i$.
We stress that these computations are performed independently for each
robot on different processes, without having access to the other robot information.
After the initialization, robots
start communicating and performing the steps of the distributed algorithm
proposed in Section~\ref{sec:alg_description}.
To implement Algorithm~\ref{alg:algorithm}, we used the \textsc{disropt}
Python package~\cite{farina2019disropt}, which provides the needed features to
encode the algorithm steps and is compatible with the \textsc{ChoiRbot} framework (see
also~\cite{testa2020choirbot}).
As soon as the distributed optimization procedure completes,
robots start moving towards the assigned tasks.
Due to constraint~\eqref{eq:PDVRP_coupling_constr},
suboptimal solutions of problem~\eqref{eq:PDVRP} may lead
more than one robot to perform the same task.
Thus, the robot communicates to a so-called
\textsc{Auth} node that it wants to start a particular task.
If the task has been already taken care of by another robot,
the \textsc{Auth} node denies authorization and the robot
performs the next one.
In such a way, redundant assignments are avoided.
The target positions are then communicated to the local trajectory
planning layer and then fed to the low-level controllers to steer
the robots over the requests positions.
The controller nodes interact with Gazebo, which simulates the robot
dynamics and provides the pose of each robot.

We have experimentally found that a satisfying tuning of the distributed algorithm
is as follows. The robots perform $250$ iterations with local allocation initialized 
as in Algorithm~\ref{alg:algorithm}.
For the first $125$ iterations they use the diminishing step size
$\alpha^t = 0.005/(t + 1)$, then they use a constant step size (equal to the last computed
one). Because of the constant step size, the final allocation fed to the thresholding
operation~\eqref{eq:alg_rounding} is the running average computed from iteration
$126$ on, i.e.
\begin{align*}
	\bigg( \sum_{\tau=126}^{250} \alpha^\tau y_i^\tau \bigg) /
	\bigg( \sum_{\tau=126}^{250} \alpha^\tau \bigg)
\end{align*}
This particular tweaking allows the robots to quickly converge to a good-quality
solution.

\subsection{Results}
We performed $3$ Monte Carlo simulations on random
instances of problem~\eqref{eq:PDVRP} on the described platform
with TurtleBot3 robots. We test the behavior of the proposed scheme when both the number of requests $|R|$ and the number of robots $N$ are varied. In this way it is possible to assess the performance of the algorithm if $|R| > N$ or if $|R| < N$.

\emph{First simulation.}
To begin with, we test optimality of the solution computed
by the algorithm while varying the number of robots $N$.
We perform $50$ Monte Carlo trials for each value of $N$
and we fix $\delta = 0.1$ to prioritize optimality over feasibility
(cf. Section~\ref{sec:discussion_parameters}).
For each trial, $10$ pickup requests and $10$
corresponding deliveries are randomly generated on the plane.
In Figure~\ref{fig:montecarlo_cost_error}, we show the cost error
of the solution actuated by robots after $250$ iterations of the
distributed algorithm (in blue),
compared to the cost of a centralized solver,
with varying number of robots. The distributed algorithm achieves an
average $30-40\%$ suboptimality.

\begin{figure}[htbp]\centering
  \hspace{2cm}
  \includegraphics[scale=1]{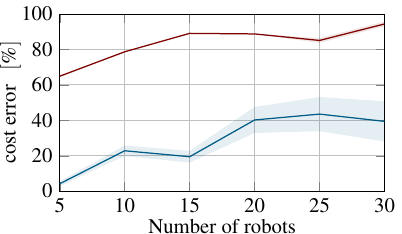}
  \caption{Cost error in Monte Carlo simulations on Gazebo for
    varying number of robots. Blue: proposed approach. Red: baseline approach. The shaded areas represent one standard deviation.}
  \label{fig:montecarlo_cost_error}
\end{figure}

\emph{Comparison with distributed baseline approach.}
The solutions are confronted
with a distributed greedy market-based algorithm as follows.
Upon receiving the task requests
and filtering out those that cannot be performed due to
insufficient load capacity, each robot initially self-assigns the
pickup task closest to its position together with the corresponding
delivery. Then, it self-assigns the next pickup task closest to the last
delivery location, and so on until the task list is empty. To
remove ties, robots compute the cost of performing each pickup-delivery
pair, and then perform a min-consensus algorithm to decide the
robot with the lower cost.
This algorithm is run up to its convergence to the best attainable value.
The cost error obtained with this approach is
depicted in Figure~\ref{fig:montecarlo_cost_error} (in red).
From the figure, it emerges that our approach outperforms the baseline
since it achieves lower suboptimality levels.

\emph{Second simulation.}
Now, we assess the behavior of the cost error while varying
the total number of requests. Specifically, we consider a team of
$20$ robots and we let the number of requests $|R|$ vary from
$4$ to $24$ (with $\delta = 0.9$).
For each of these values of $|R|$, we perform $50$ trials
and we let the robots implement the solution after $250$ iterations
of the distributed algorithm.
The results are depicted in Figure~\ref{fig:task_cost_error} and
Figure~\ref{fig:cost_comparison}.
Notably, as it can be seen from Figure~\ref{fig:task_cost_error},
the mean relative error remains constant while increasing the number
of requests. This is an appealing feature of the proposed strategy
considering the fact that, as depicted also in Figure~\ref{fig:cost_comparison},
the global optimal cost increases with the number of tasks.

\begin{figure}[htbp]\centering
  \includegraphics[scale=1]{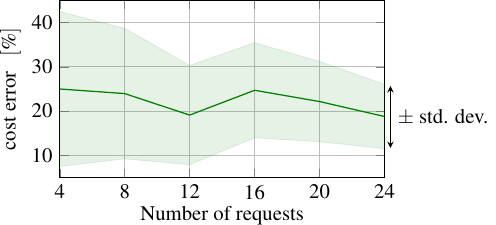}
  \caption{Cost error in Monte Carlo simulations on Gazebo for
    varying number of requests.}
  \label{fig:task_cost_error}
\end{figure}
\begin{figure}[htbp]\centering
  \includegraphics[scale=1]{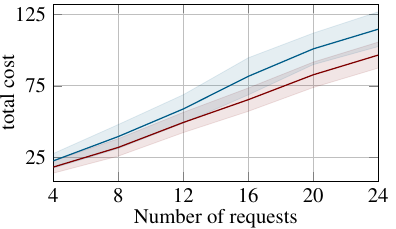}
  \caption{Comparison between the centralized optimal solution (red) and the one found by the proposed distributed strategy (blue).}
  \label{fig:cost_comparison}
\end{figure}

\emph{Third simulation.}
Finally, we perform simulations to determine the
number of iterations needed to achieve finite-time feasibility
while varying the number of robots $N$ and the value of $\delta$.
We employed three values of $\delta$, namely $0.1$, $0.5$,
$0.9$, and performed $50$ trials for each of the values
$N = 5, 10, 15, 20, 25, 30$ and for each value of $\delta$.
In each trial, we performed $250$ iterations of the algorithm
and recorded the value of the coupling constraint.
Then we determined the following quantity
\begin{align*}
  \min \: & \: t
  \\
  \text{such that} \: & \: \sum_{i=1}^N \sum_{k:(j,k) \in \EE_A} {x_i^{jk}}^\tau \geq 1
    \hspace{1cm} \text{for all } \tau \ge t,
\end{align*}
which is essentially an empirical value of $T_\delta$ appearing
in Theorem~\ref{thm:finite_time_feasibility}.
Interestingly, we found out that, for all the trials, the empirical
value of $T_\delta$ is zero, which means that in all the simulated
scenarios the algorithm provides a feasible solution to the
PDVRP~\eqref{eq:PDVRP} since the first iteration.

\section{Experiments}
\label{sec:experiments}

To conclude, we show experimental results on real teams of robots solving
PDVRP instances. We first performed a small benchmark experiment to assess
the performance of the algorithm and then we perform more complex
experiments to showcase how it can be implemented on large fleets of robots.

\subsection{Benchmark Experiment}
In this experiment, we consider a fleet of $4$ TurtleBot3 burger ground robots that have to serve $4$ pickup tasks and their corresponding $4$ deliveries. In this set-up, robots navigate in a cluttered environment containing obstacles. Robots start from initial positions on a line. Pickup and delivery tasks are generated on two different lines, so as to clearly distinguish the optimal robot-to-task assignment. To simulate the pickup/delivery procedure, each robot waits over the request location a random service time $d^j$ between $3$ and $5$ seconds. The capacity
$C_i$ of each robot and the demand/supply $q^j$ of tasks are drawn
from uniform distributions. The velocity of robots is approximately $0.2\, \mathrm{m}/\mathrm{s}$.
As regards the low-level controllers, we use a linear state feedback
for single integrators. In this way, we can handle collision among
robots and with obstacles via barrier functions using the approach described in~\cite{wilson2020robotarium}.
Then, in order to get the unicycle inputs, we utilize a near-identity
diffeomorphism (see~\cite{wilson2020robotarium}).

In Figure~\ref{fig:pdvrp_experiment_benchmark}, we show a snapshot of the experiment set-up.
We run the distributed algorithm for $1000$ iterations and record the robot-to-task assignment and the cost along the algorithm evolution.
Figure~\ref{fig:pdvrp_experiment_benchmark_analysis} reports the experiment analysis, while in Table~\ref{tb:results}, we include a comparison of the allocation found by the proposed
distributed algorithm and the one found via a centralized solver.
The figure highlights that, as the algorithm evolves, the cost of the overall robot-to-task assignment decreases, and after a certain number of iterations the computed solution becomes conflict free (i.e., only one robot is assigned to each task). As it can be seen from the table, the solution computed by the distributed algorithm coincides with the optimal (centralized) solution.
A video is also available as supplementary material to the paper.\footnote{The video can be also found at \texttt{\url{https://youtu.be/tdtGNftdbng}}.}

\begin{figure}[htbp]
\centering
 \includegraphics[width=.7\columnwidth]{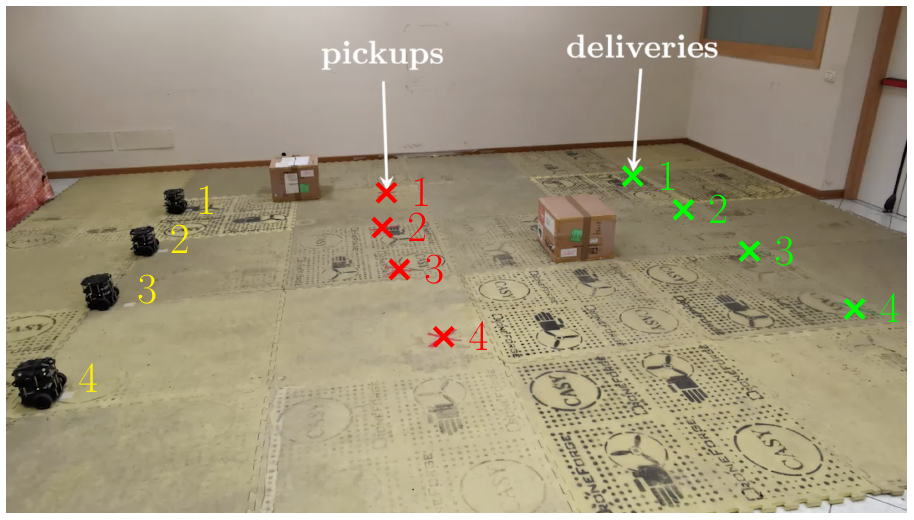}
   \caption{Benchmark experiment for the \ac{PDVRP} problem. Robots start on a line. Pickups locations are denoted with red crosses, while deliveries are denoted with green crosses.}
  \label{fig:pdvrp_experiment_benchmark}
\end{figure}

\begin{figure}[htbp]
\centering

\includegraphics[width=.8\columnwidth]{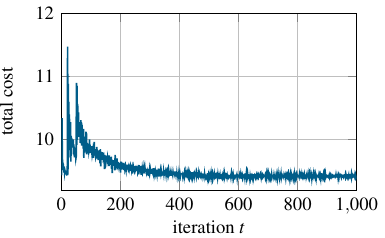}

\vspace{0.2cm}

\includegraphics[width=.8\columnwidth]{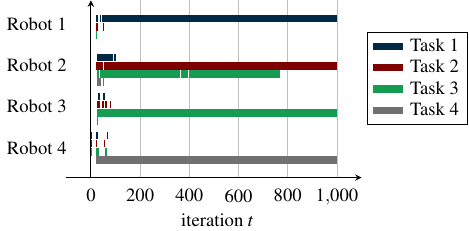}

   \caption{Analysis of benchmark experiment. Top: value of the cost along the algorithmic evolution, computed from problem~\eqref{eq:alg_distributed_z_LP}. Bottom: assignment of robot to tasks along the algorithmic evolution.}
  \label{fig:pdvrp_experiment_benchmark_analysis}
\end{figure}

\begin{table}[h]
\footnotesize
	\begin{center}
	\caption{Experiment analysis: final assignments}\label{tb:results}
\begin{tabular}{c|c|c}
 Robot id & Distributed solution & Optimal solution \\
 \hline
 1 & [1]    & [1]\\
 2 & [2]    & [2]\\
 3 & [3]    & [3]\\
 4 & [4]    & [4]\\
 \hline
\end{tabular}
\end{center}
\end{table}

\subsection{Large Experiments}
We consider heterogeneous teams composed 
by Crazyflie nano-quadrotors
and TurtleBot3 Burger mobile robots.
Tasks are generated randomly in the space. In particular, we split the
experiment area in two halves. Pickup requests are located in the right half,
while deliveries are in the left half. 
Each robot can serve a subset of the pickup/delivery requests. This is
decided randomly at the beginning of the experiment. The velocity of robots (both ground and aerial)
is approximately $0.2\, \mathrm{m}/\mathrm{s}$.
The solution mechanism is the same used in the
simulations.

As regards the low-level controllers of nano-quadrotors,
a hierarchical controller has been considered.
Specifically, a flatness-based position controller generates desired angular
rates that are then actuated with a low-level PID control loop.
The position controller receives as input a sufficiently smooth position
trajectory, which is computed as a polynomial spline.

We performed two different experiments. In the first one, there
are $3$ ground robots and $2$ aerial robots that must serve a total of
$5$ pickups and $5$ deliveries. In Figure~\ref{fig:pdvrp_experiment_1},
we show a snapshot from the experiment.
Then we performed a second, larger experiment with $7$ ground robots and
$2$ aerial robots that must serve $10$ pickups and $10$ deliveries.
In Figure~\ref{fig:pdvrp_experiment_2}, we show snapshots from the second experiment.
A video is also available as supplementary material to the paper.\footnote{The video can be also found at \texttt{\url{https://youtu.be/NwqzIEBNIS4}}.}

\begin{figure}[htbp]
\centering
 \includegraphics[width=.9\columnwidth]{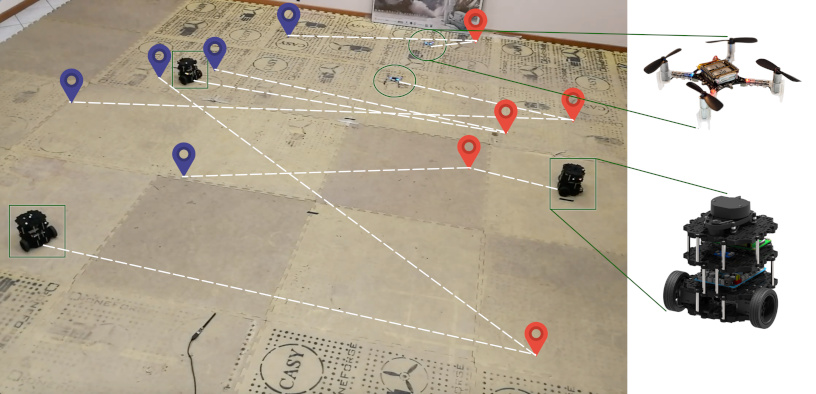}
   \caption{First experiment with ground and aerial robots for the
    \ac{PDVRP} problem.
    Ground robots are indicated with
    a square, while aerial robots are delimited with circles.
    The red pins represent pickups, while the blue ones represent deliveries.
    The paths travelled by robots are depicted as dashed lines.}
  \label{fig:pdvrp_experiment_1}
\end{figure}

\begin{figure}[htbp]
\centering
   \includegraphics[width=.49\columnwidth]{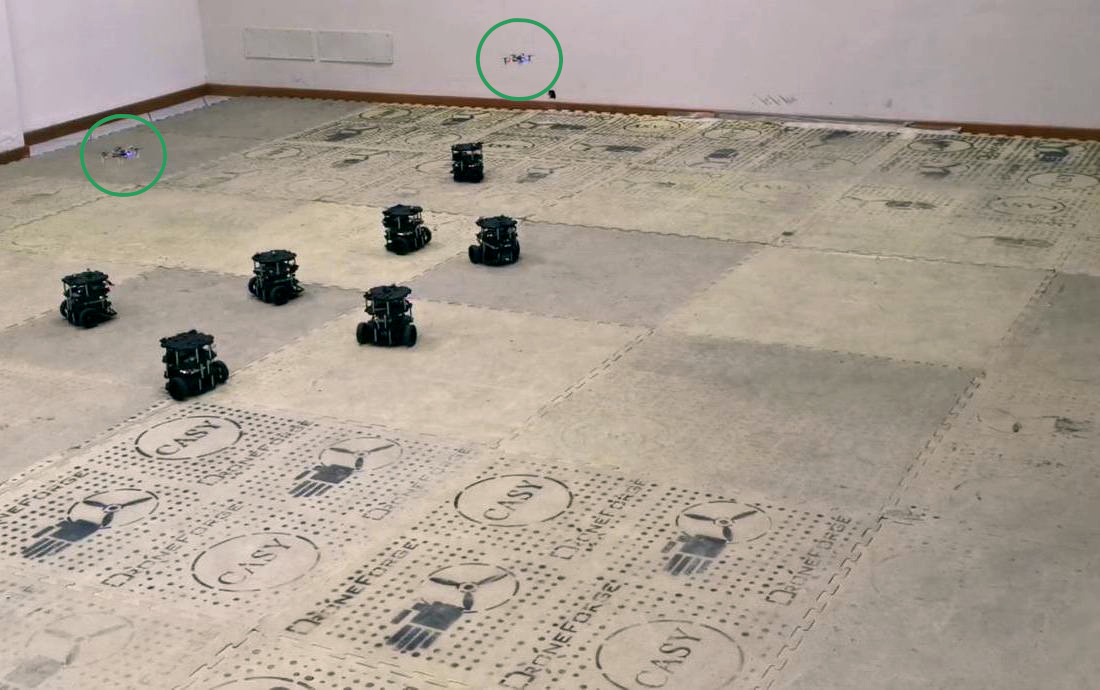}\hfill
   \includegraphics[width=.49\columnwidth]{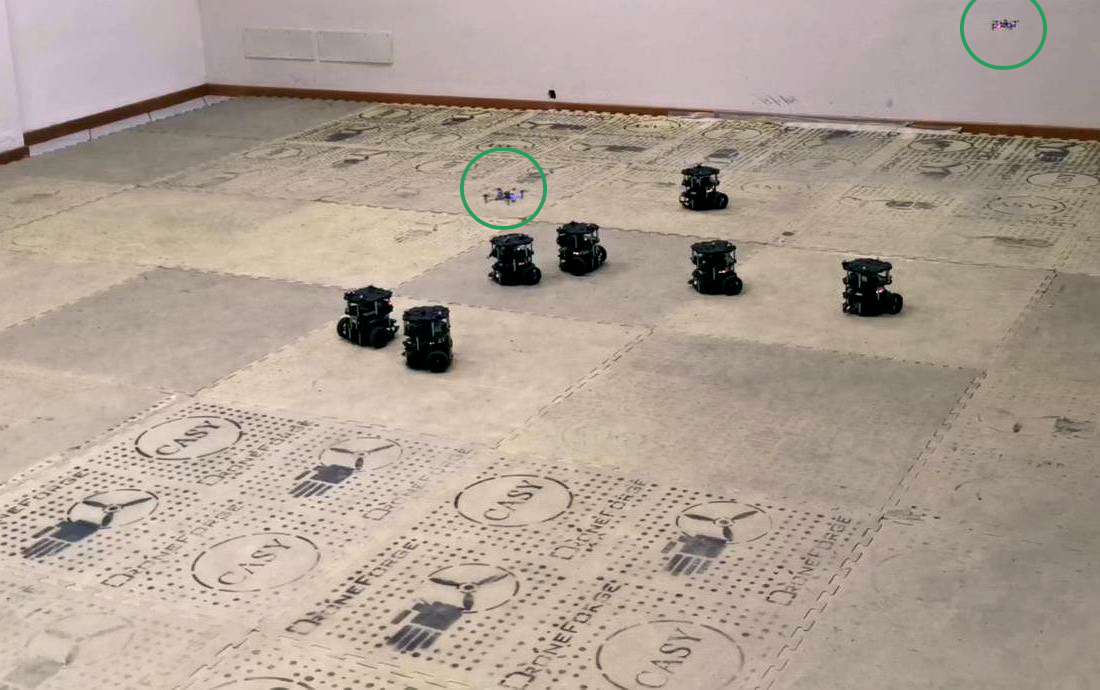}
   \caption{Snapshots from the second experiment. Left: robots have reached
     the pickup positions and perform the loading operation (simulated).
     Right: robots have reached the delivery positions and have terminated
     the mission.}
  \label{fig:pdvrp_experiment_2}
\end{figure}

\section{Conclusions}
\label{sec:conclusions}

In this paper, we introduced a purely distributed scheme to address
large-scale instances of the Pickup-and-Delivery Vehicle Routing Problem
in networks of cooperating robots. %
The proposed distributed algorithm
is shown to provide a feasible solution in a finite number of communication
rounds and has remarkable scalability and privacy-preserving properties
that allow for large robotic networks. %
The theoretical results are corroborated on a set of %
realistic simulations %
through a combined ROS~2 / Gazebo platform.
Finally, experimental results on a real testbed highlight feasibility of the
proposed solution for a heterogeneous team of ground and aerial robots.
As a future line of research, a more complex set-up could be investigated. Indeed, variables as travel cost, service duration, travel time could be considered to be stochastic in order to obtain a more realistic model. Moreover, the approach can be enhanced to force assignments of single robots to each task and a convergence rate analysis can be performed.

\appendices
\section{Conversion to Mixed-Integer Linear Program}\label{sec:conv_lin}

Problem~\eqref{eq:PDVRP} is almost a mixed-integer linear program,
except for the fact that the constraints~\eqref{eq:PDVRP_local_nonlcon_1} and~\eqref{eq:PDVRP_local_nonlcon_2}
are nonlinear. However, from a computational point of view, the
constraints~\eqref{eq:PDVRP_local_nonlcon_1} and~\eqref{eq:PDVRP_local_nonlcon_2} can be
readily recast as linear ones.
To achieve this, we use a standard procedure
(see~\cite{cordeau2006branch,bemporad1999control}) that can be
summarized as follows.

First, we introduce for each $i,j,k$ the constraints
$0 \leq B_i^j\leq \overline{B}_i$, where $\overline{B}_i \ge 0$ is
any conservative upper bound on the total travel time of vehicle $i$
selected so as to preserve the solutions of the original problem~\footnote{A simple possibility is to select $\overline{B}_i$ as the sum
of all possible travel times $t_{i}^{jk}$ from all $i$ to all $j$,
plus the service times $d^j$ for all $j$.}.
While this operation does not affect the problem, it introduces a bound
on the value of each $B^j$. After defining for all $i,j,k$ the scalars $M_{i}^{jk} = \overline{B}_i + d^j + t_i^{jk}$
(or any larger number), the nonlinear constraint~\eqref{eq:PDVRP_local_nonlcon_1}
can be replaced with the linear one
\begin{align}
  B_i^k \ge B_i^j+d_j+t_i^{jk} - M_{i}^{jk} (1 - x_i^{jk}).
\label{eq:linearized_constraint_B}
\end{align}
Equivalence of the constraint~\eqref{eq:PDVRP_local_nonlcon_1} with~\eqref{eq:linearized_constraint_B}
can be verified by noting that, for $x_i^{jk} = 1$,
we obtain the desired constraint $B_i^k \ge B_i^j+d_j+t_i^{jk}$, while for $x_i^{jk} = 0$
the constraint~\eqref{eq:linearized_constraint_B} becomes $B_i^k \ge B_i^j+d_j+t_i^{jk} - M_{i}^{jk}$
(which is already implied by the constraints $B_i^k \ge 0$ and $B_i^j \ge 0$).

A similar reasoning can be applied to turn the constraint~\eqref{eq:PDVRP_local_nonlcon_2} into a linear one.
Let us define $\overline{W}_i^{jk} = \overline{Q}^j + q^k$ and
$\underline{W}_i^{jk} = \overline{Q}^k - q^k - \underline{Q}^j$
(or any larger number), then constraint~\eqref{eq:PDVRP_local_nonlcon_2} can be
equivalently replaced with the pair of linear constraints
\begin{subequations}
\label{linearized_constraint_Q}
\begin{align}
  Q_i^k &\ge Q_i^j+q^k - \overline{W}_{i}^{jk} (1 - x_i^{jk}),
  \label{linearized_constraint_Q_1}
  \\
  Q_i^k &\le Q_i^j + q^k + \underline{W}_{i}^{jk} (1 - x_i^{jk}).
  \label{linearized_constraint_Q_2}
\end{align}
\end{subequations}
After introducing the additional constraints $0 \leq B_i^j\leq \overline{B}_i$
for all $i,j,k$ and replacing~\eqref{eq:PDVRP_local_nonlcon_1}--\eqref{eq:PDVRP_local_nonlcon_2}
with their equivalent versions~\eqref{eq:linearized_constraint_B}--\eqref{linearized_constraint_Q},
problem~\eqref{eq:PDVRP} becomes a MILP.

\section{Primal Decomposition}
\label{sec:primal_decomp}
Consider a network of $N$ agents indexed by $\agents = \until{N}$
that aim to solve a linear program of the form
\begin{align}
\begin{split}
  \min_{x_1, \ldots, x_N} \: & \: \sum_{i=1}^N c_i^\top x_i
  \\
  \subj \: & \: x_i \in X_i, \hspace{1cm} \forall i \in \agents,
  \\
  & \: \sum_{i=1}^N A_i x_i \le b,
\end{split}
\label{eq:app_LP}
\end{align}
where each $x_i \in \real^{n_i}$ is the $i$-th optimization variable,
$c_i \in \real^{n_i}$ is the $i$-th cost vector, $X_i \subset \real^{n_i}$
is the $i$-th polyhedral constraint set and $A_i \in \real^{S \times n_i}$
is a matrix for the $i$-th contribution to the \emph{coupling constraint}
$\sum_{i=1}^N A_i x_i \le b \in \real^S$. Problem~\eqref{eq:app_LP}
enjoys the constraint-coupled structure~\cite{notarstefano2019distributed}
and can be recast into a master-subproblem architecture
by using the so-called \emph{primal decomposition} technique~\cite{silverman1972primal}.
The right-hand side vector $b$ of the coupling constraint
is interpreted as a given (limited) resource to be shared among the
network agents.
Thus, local \emph{allocation vectors} $y_i \in \real^S$ for all $i$
are introduced such that $\sum_{i=1}^N y_i = b$.
To determine the allocations, a \emph{master problem} is introduced
\begin{align}
\begin{split}
  \min_{y_1,\ldots,y_N} \: & \: \sum_{i =1}^N p_i (y_i) 
  \\
  \subj \: & \: \smallsum_{i=1}^N y_i = b
  \\
  & \: y_i \in Y_i, \hspace{1cm} \forall i \in\agents,
\end{split}
\label{eq:app_primal_decomp_master}
\end{align}
where, for each $i\in\agents$, the function $\map{p_i}{\real^S}{\real}$ is
defined as the optimal cost of the $i$-th (linear programming) \emph{subproblem}
\begin{align}
\begin{split}
  p_i(y_i) = \: \min_{x_i} \: & \: c_i^\top x_i
  \\
  \subj \: 
  & \: A_i x_i \leq y_i
  \\
  & \: x_i \in X_i.
\end{split}
\label{eq:app_primal_decomp_subproblem}
\end{align}
In problem~\eqref{eq:app_primal_decomp_master}, the new constraint set
$Y_i \subseteq\real^S$ is
the set of $y_i$ for which
problem~\eqref{eq:app_primal_decomp_subproblem} is feasible, i.e.,
such that there exists $x_i \in X_i$ satisfying the local
\emph{allocation constraint} $A_i x_i \le y_i$.
Assuming problem~\eqref{eq:app_LP} is feasible and $X_i$ are compact sets,
if $(y_1^\star, \ldots, y_N^\star)$ is an optimal solution
of~\eqref{eq:app_primal_decomp_master} and, for all $i$,
$x_i^\star$ is optimal for~\eqref{eq:app_primal_decomp_subproblem}
(with $y_i = y_i^\star$), then $(x_1^\star, \ldots, x_N^\star)$
is an optimal solution of the original problem~\eqref{eq:app_LP}
(see, e.g.,~\cite[Lemma 1]{silverman1972primal}).

\section{Proofs}

\subsection{Proof of Lemma~\ref{lemma:well_posed_conv}}
\label{sec:proof_lemma_well_posed_conv}
  Note that problem~\eqref{eq:alg_distributed_z_LP} is the epigraph form of
  \begin{align*}
  \begin{split}      
    \min_{x_i, B_i, Q_i} \: &\: \sum_{(j,k) \in \EE_A} c_i^{jk} x_i^{jk} 
    \\
    & \hspace{1cm} + M \max\bigg\{ 0,  \max_{j \in \Rset} \bigg( [y_i^t]_j - \!\!\sum_{k:(j,k) \in \EE_A} \!\! x_i^{jk} \bigg) \!\bigg\}
    \\
    \subj \: & \: (x_i, B_i, Q_i) \in \conv{Z_i}
  \end{split}
  \end{align*}
  Moreover, it holds $\conv{Z_i} \supset Z_i$.
  Therefore, the proof follows since $Z_i$ is not empty (by assumption).
  \oprocend

\subsection{Proof of Lemma~\ref{lemma:well_posed_milp}}
\label{sec:proof_lemma_well_posed_milp}
  Fix a robot $i$. Because of the thresholding operation~\eqref{eq:alg_rounding},
  it holds $[\yend_i]_j \le 1$ for all $j \in \Rset$.
  We now show that the feasible set of problem~\eqref{eq:alg_MILP} is not empty.
  Since problem~\eqref{eq:PDVRP} is assumed to be feasible,
  we know that
  there exists $z_i = (x_i, B_i, Q_i) \in Z_i$. Thus, we only have to show
  that the constraint $\sum_{k:(j,k) \in \EE_A} x_i^{jk} \geq [\yend_i]_j$ for all
  $j \in \Rset$ can be satisfied by at least one vector $z_i \in Z_i$.
  
  Due to the flow constraints~\eqref{eq:PDVRP_local_con3}, the subtour elimination
  constraints~\eqref{eq:PDVRP_local_nonlcon_1} and the integer
  constraints~\eqref{eq:PDVRP_local_con5}, for all $j \in \Rset$ the quantity
  $\sum_{k:(j,k) \in \EE_A} x_i^{jk}$ is either equal to $0$ or equal to $1$,
  since there can be at most one index $k$ satisfying $(j,k) \in \EE_A$
  and such that $x_i^{jk} = 1$.
  Consider the constraint $\sum_{k:(j,k) \in \EE_A} x_i^{jk} \geq [\yend_i]_j$
  and fix a component $j \in \Rset$. Note that this constraint essentially
  imposes whether or not robot $i$ must pass through location $j$.
  Indeed, on the one hand, if $[\yend_i]_j \le 0$, the vector $z_i$ can be chosen
  such that the left-hand side is either equal to $0$ (vehicle $i$ does not pass
  through location $j$) or equal to $1$ (vehicle $i$ passes through location $j$).
  In either case, it holds $\sum_{k:(j,k) \in \EE_A} x_i^{jk} \ge 0 \ge [\yend_i]_j$
  so that the constraint is satisfied.
  On the other hand, if $0 < [\yend_i]_j \le 1$ (recall that $[\yend_i] \le 1$ for all
  $j \in \Rset$ and thus
  there are no other possibilities), then the only way to satisfy the constraint is to
  have $x_i^{jk} = 1$ for some index $k$ with $(j,k) \in \EE_A$, in which this case
  we would obtain $1 = \sum_{k:(j,k) \in \EE_A} x_i^{jk} \ge [\yend_i]_j > 0$.
  As a consequence, problem~\eqref{eq:alg_MILP} admits as feasible solution any vector
  $(x_i, B_i, Q_i) \in Z_i$ representing a path passing through all the locations
  $j \in \Rset$ and satisfying $[\yend_i]_j > 0$.
  \oprocend


%

%

\begin{IEEEbiography}
  [{\includegraphics[width=2.5cm]{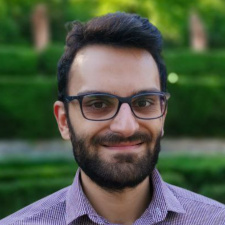}}]{Andrea Camisa}
received the Laurea degree summa cum laude in Computer Engineering from the University of Salento, Italy in 2017, the Licenza degree from the ISUFI excellence school, Italy in 2018 and the Ph.D degree in ``Biomedical, Electrical, and Systems Engineering'' from the University of Bologna, Italy in 2021.

He is a Postdoctoral Research Fellow and Adjunct Professor at the University of Bologna, Italy. He was a visiting student at the University of Stuttgart in 2017 and 2018. His research interests include reinforcement learning, convex, distributed and mixed-integer optimization.
\end{IEEEbiography}

\begin{IEEEbiography}
  [{\includegraphics[width=2.5cm]{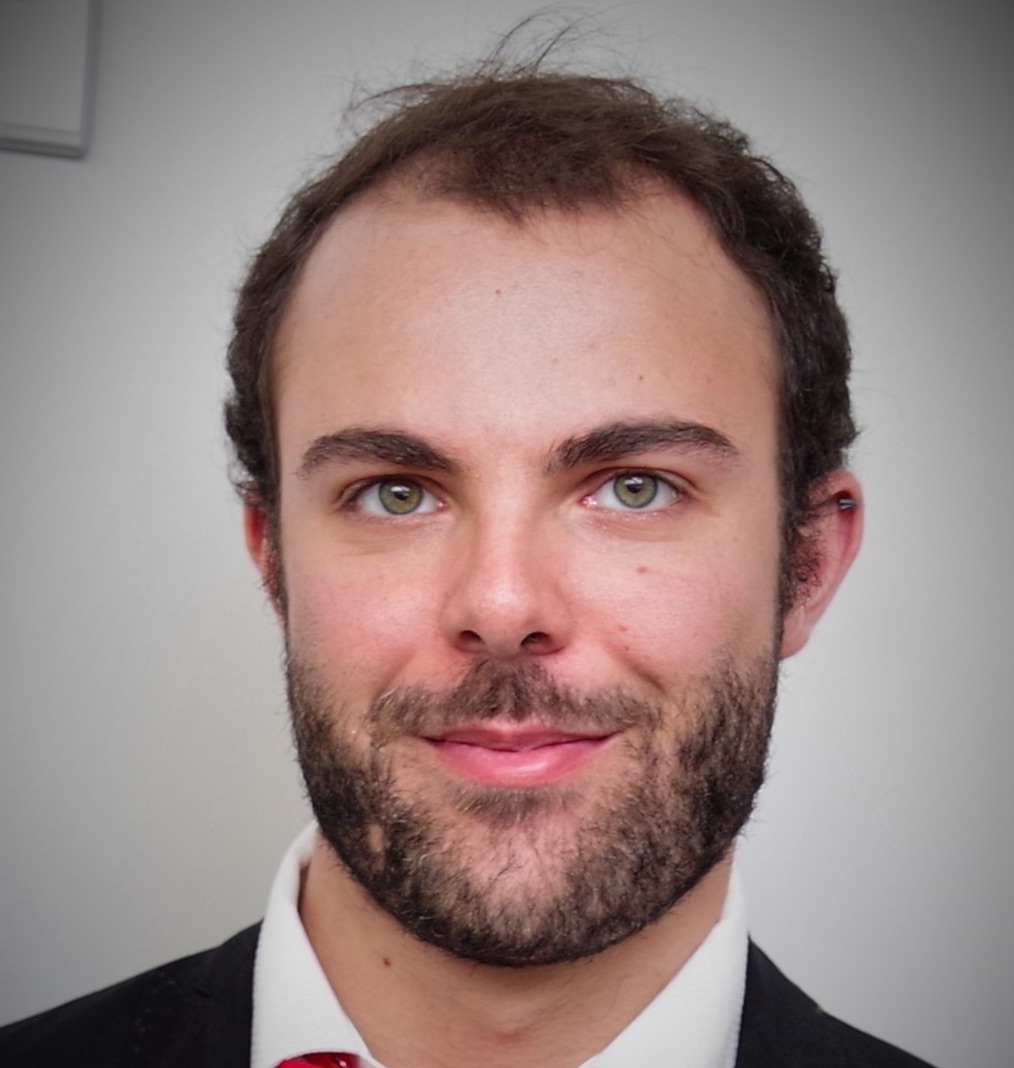}}]{Andrea Testa}
received the Laurea degree ``summa cum laude'' in Computer Engineering 
 rom the Universit\`a del Salento, Lecce, Italy in 2016
and the Ph.D degree %
in Engineering of Complex Systems 
from the same university in 2020.

He is a Research Fellow at Alma Mater Studiorum Universit\`a di Bologna, Bologna, Italy.
He was a visiting scholar at LAAS-CNRS, Toulouse, (July to September 2015 and February 2016) and at Alma Mater Studiorum Universit\`a di Bologna (October 2018 to June 2019).
His research interests include control of UAVs and distributed optimization.
\end{IEEEbiography}

\begin{IEEEbiography}
  [{\includegraphics[width=2.5cm]{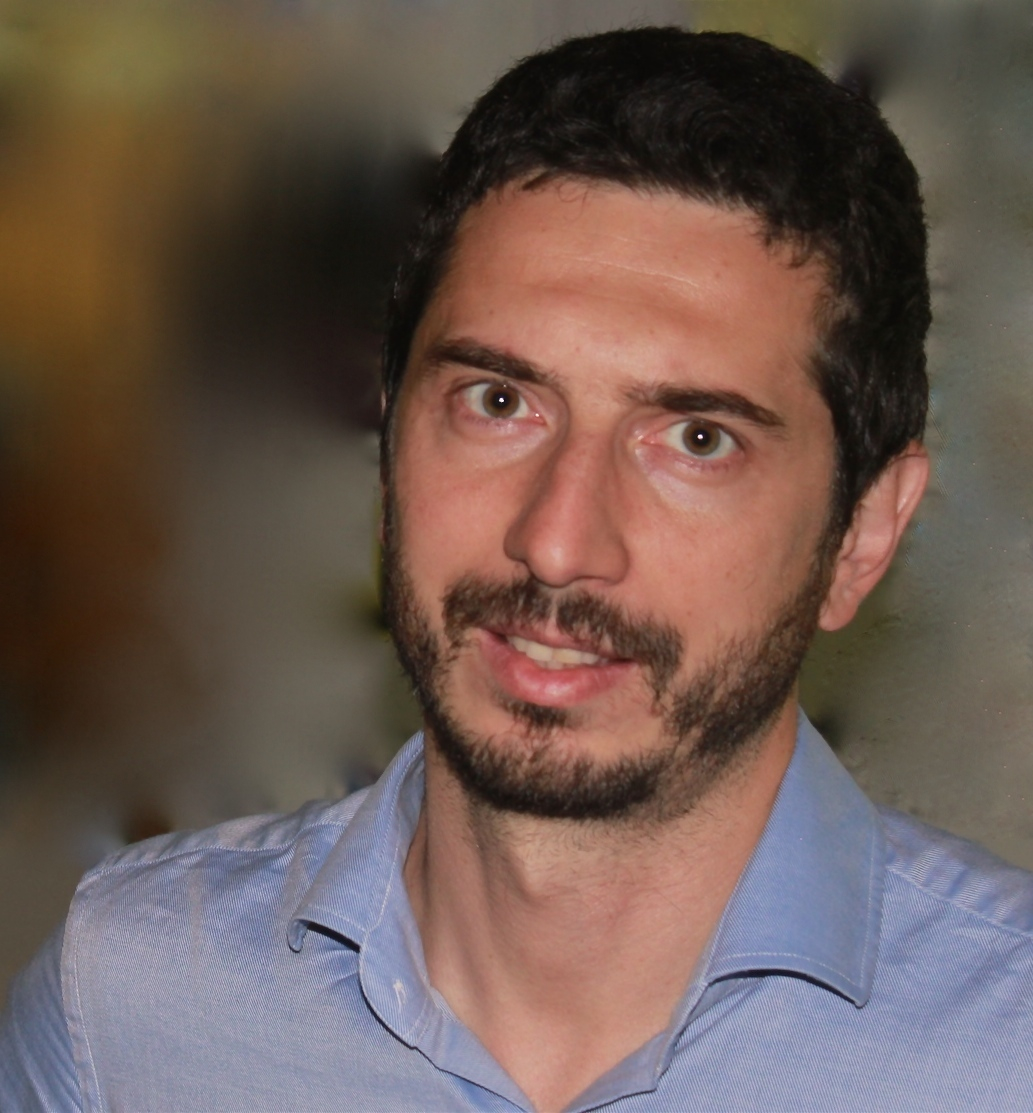}}] {Giuseppe Notarstefano}
received the Laurea degree summa cum laude in electronics engineering from the Universit\`a di Pisa, Pisa, Italy, in 2003 and the Ph.D. degree in automation and operation research from the Universit\`a di Padova, Padua, Italy, in 2007.

He is a Professor with the Department of Electrical, Electronic, and Information Engineering G. Marconi, Alma Mater Studiorum Universit\`a di Bologna, Bologna, Italy. He was Associate Professor (from June 2016 to June 2018) and
previously Assistant Professor, Ricercatore (from February 2007), with the Universit\`a del Salento, Lecce, Italy. He has been Visiting Scholar at the University of Stuttgart, University of California Santa Barbara, Santa Barbara, CA, USA and University of Colorado Boulder, Boulder, CO, USA. His research interests include distributed optimization, cooperative control in complex networks, applied nonlinear optimal control, and trajectory optimization and maneuvering of aerial and car vehicles.

Dr. Notarstefano serves as an Associate Editor for \emph{IEEE Transactions on Automatic Control}, \emph{IEEE Transactions on Control Systems Technology}, and \emph{IEEE Control Systems Letters}. He has been also part of the Conference Editorial Board of IEEE Control Systems Society and EUCA. He is recipient of an ERC Starting Grant 2014.
\end{IEEEbiography}


\begin{thebibliography}{10}
\providecommand{\url}[1]{#1}
\csname url@samestyle\endcsname
\providecommand{\newblock}{\relax}
\providecommand{\bibinfo}[2]{#2}
\providecommand{\BIBentrySTDinterwordspacing}{\spaceskip=0pt\relax}
\providecommand{\BIBentryALTinterwordstretchfactor}{4}
\providecommand{\BIBentryALTinterwordspacing}{\spaceskip=\fontdimen2\font plus
\BIBentryALTinterwordstretchfactor\fontdimen3\font minus
  \fontdimen4\font\relax}
\providecommand{\BIBforeignlanguage}[2]{{%
\expandafter\ifx\csname l@#1\endcsname\relax
\typeout{** WARNING: IEEEtran.bst: No hyphenation pattern has been}%
\typeout{** loaded for the language `#1'. Using the pattern for}%
\typeout{** the default language instead.}%
\else
\language=\csname l@#1\endcsname
\fi
#2}}
\providecommand{\BIBdecl}{\relax}
\BIBdecl

\bibitem{kamra2017combinatorial}
N.~Kamra, T.~S. Kumar, and N.~Ayanian, ``Combinatorial problems in multirobot
  battery exchange systems,'' \emph{IEEE Transactions on Automation Science and
  Engineering}, vol.~15, no.~2, pp. 852--862, 2017.

\bibitem{ham2019drone}
A.~Ham, ``Drone-based material transfer system in a robotic mobile fulfillment
  center,'' \emph{IEEE Transactions on Automation Science and Engineering},
  vol.~17, no.~2, pp. 957--965, 2019.

\bibitem{gombolay2018fast}
M.~C. Gombolay, R.~J. Wilcox, and J.~A. Shah, ``Fast scheduling of robot teams
  performing tasks with temporospatial constraints,'' \emph{IEEE Transactions
  on Robotics}, vol.~34, no.~1, pp. 220--239, 2018.

\bibitem{bai2019efficient}
X.~Bai, M.~Cao, W.~Yan, and S.~S. Ge, ``Efficient routing for
  precedence-constrained package delivery for heterogeneous vehicles,''
  \emph{IEEE Transactions on Automation Science and Engineering}, vol.~17,
  no.~1, pp. 248--260, 2019.

\bibitem{toth2002vehicle}
P.~Toth and D.~Vigo, \emph{The vehicle routing problem}.\hskip 1em plus 0.5em
  minus 0.4em\relax SIAM, 2002.

\bibitem{parragh2008survey}
S.~N. Parragh, K.~F. Doerner, and R.~F. Hartl, ``A survey on pickup and
  delivery models part {II}: Transportation between pickup and delivery
  locations,'' \emph{Journal f{\"u}r Betriebswirtschaft}, vol.~58, no.~2, pp.
  81--117, 2008.

\bibitem{pillac2013review}
V.~Pillac, M.~Gendreau, C.~Gu{\'e}ret, and A.~L. Medaglia, ``A review of
  dynamic vehicle routing problems,'' \emph{European Journal of Operational
  Research}, vol. 225, no.~1, pp. 1--11, 2013.

\bibitem{ritzinger2016survey}
U.~Ritzinger, J.~Puchinger, and R.~F. Hartl, ``A survey on dynamic and
  stochastic vehicle routing problems,'' \emph{International Journal of
  Production Research}, vol.~54, no.~1, pp. 215--231, 2016.

\bibitem{coltin2013online}
B.~Coltin and M.~Veloso, ``Online pickup and delivery planning with transfers
  for mobile robots,'' in \emph{2014 IEEE International Conference on Robotics
  and Automation (ICRA)}.\hskip 1em plus 0.5em minus 0.4em\relax IEEE, 2014,
  pp. 5786--5791.

\bibitem{liu2019task}
M.~Liu, H.~Ma, J.~Li, and S.~Koenig, ``Task and path planning for multi-agent
  pickup and delivery.'' in \emph{AAMAS}, 2019, pp. 1152--1160.

\bibitem{fauadi2013intelligent}
M.~H. F. b.~M. Fauadi, S.~H. Yahaya, and T.~Murata, ``Intelligent combinatorial
  auctions of decentralized task assignment for agv with multiple loading
  capacity,'' \emph{IEEJ Transactions on electrical and electronic
  Engineering}, vol.~8, no.~4, pp. 371--379, 2013.

\bibitem{heap2013repeated}
B.~Heap and M.~Pagnucco, ``Repeated sequential single-cluster auctions with
  dynamic tasks for multi-robot task allocation with pickup and delivery,'' in
  \emph{German Conference on Multiagent System Technologies}.\hskip 1em plus
  0.5em minus 0.4em\relax Springer, 2013, pp. 87--100.

\bibitem{arsie2009efficient}
A.~Arsie, K.~Savla, and E.~Frazzoli, ``Efficient routing algorithms for
  multiple vehicles with no explicit communications,'' \emph{IEEE Transactions
  on Automatic Control}, vol.~54, no.~10, pp. 2302--2317, 2009.

\bibitem{soeanu2011decentralized}
A.~Soeanu, S.~Ray, M.~Debbabi, J.~Berger, A.~Boukhtouta, and A.~Ghanmi, ``A
  decentralized heuristic for multi-depot split-delivery vehicle routing
  problem,'' in \emph{2011 IEEE International Conference on Automation and
  Logistics (ICAL)}.\hskip 1em plus 0.5em minus 0.4em\relax IEEE, 2011, pp.
  70--75.

\bibitem{chopra2017distributed}
S.~Chopra, G.~Notarstefano, M.~Rice, and M.~Egerstedt, ``A distributed version
  of the hungarian method for multirobot assignment,'' \emph{IEEE Transactions
  on Robotics}, vol.~33, no.~4, pp. 932--947, 2017.

\bibitem{settimi2013subgradient}
A.~Settimi and L.~Pallottino, ``A subgradient based algorithm for distributed
  task assignment for heterogeneous mobile robots,'' in \emph{IEEE Conference
  on Decision and Control (CDC)}, 2013, pp. 3665--3670.

\bibitem{burger2012distributed}
M.~B{\"u}rger, G.~Notarstefano, F.~Bullo, and F.~Allg{\"o}wer, ``A distributed
  simplex algorithm for degenerate linear programs and multi-agent
  assignments,'' \emph{Automatica}, vol.~48, no.~9, pp. 2298--2304, 2012.

\bibitem{testa2020generalized}
A.~Testa and G.~Notarstefano, ``Generalized assignment for multi-robot systems
  via distributed branch-and-price,'' \emph{IEEE Transactions on Robotics},
  2021.

\bibitem{luo2015distributed}
L.~Luo, N.~Chakraborty, and K.~Sycara, ``Distributed algorithms for multirobot
  task assignment with task deadline constraints,'' \emph{IEEE Transactions on
  Automation Science and Engineering}, vol.~12, no.~3, pp. 876--888, 2015.

\bibitem{testa2019distributed}
A.~Testa, A.~Rucco, and G.~Notarstefano, ``Distributed mixed-integer linear
  programming via cut generation and constraint exchange,'' \emph{IEEE
  Transactions on Automatic Control}, vol.~65, no.~4, pp. 1456--1467, 2019.

\bibitem{talebpour2019adaptive}
Z.~Talebpour and A.~Martinoli, ``Adaptive risk-based replanning for human-aware
  multi-robot task allocation with local perception,'' \emph{IEEE Robotics and
  Automation Letters}, vol.~4, no.~4, pp. 3790--3797, 2019.

\bibitem{buckman2019partial}
N.~Buckman, H.-L. Choi, and J.~P. How, ``Partial replanning for decentralized
  dynamic task allocation,'' in \emph{AIAA Scitech 2019 Forum}, 2019, p. 0915.

\bibitem{pavone2009stochastic}
M.~Pavone, N.~Bisnik, E.~Frazzoli, and V.~Isler, ``A stochastic and dynamic
  vehicle routing problem with time windows and customer impatience,''
  \emph{Mobile Networks and Applications}, vol.~14, no.~3, pp. 350--364, 2009.

\bibitem{pavone2010adaptive}
M.~Pavone, E.~Frazzoli, and F.~Bullo, ``Adaptive and distributed algorithms for
  vehicle routing in a stochastic and dynamic environment,'' \emph{IEEE
  Transactions on automatic control}, vol.~56, no.~6, pp. 1259--1274, 2010.

\bibitem{bullo2011dynamic}
F.~Bullo, E.~Frazzoli, M.~Pavone, K.~Savla, and S.~L. Smith, ``Dynamic vehicle
  routing for robotic systems,'' \emph{Proceedings of the IEEE}, vol.~99,
  no.~9, pp. 1482--1504, 2011.

\bibitem{farinelli2020decentralized}
A.~Farinelli, A.~Contini, and D.~Zorzi, ``Decentralized task assignment for
  multi-item pickup and delivery in logistic scenarios,'' in \emph{Proceedings
  of the 19th International Conference on Autonomous Agents and MultiAgent
  Systems}, 2020, pp. 1843--1845.

\bibitem{abbatecola2018distributed}
L.~Abbatecola, M.~P. Fanti, G.~Pedroncelli, and W.~Ukovich, ``A distributed
  cluster-based approach for pick-up services,'' \emph{IEEE Transactions on
  Automation Science and Engineering}, vol.~16, no.~2, pp. 960--971, 2018.

\bibitem{saleh2012mechanism}
M.~Saleh, A.~Soeanu, S.~Ray, M.~Debbabi, J.~Berger, and A.~Boukhtouta,
  ``Mechanism design for decentralized vehicle routing problem,'' in
  \emph{Proceedings of the 27th Annual ACM Symposium on Applied Computing},
  2012, pp. 749--754.

\bibitem{falsone2018distributed}
A.~Falsone, K.~Margellos, and M.~Prandini, ``A distributed iterative algorithm
  for multi-agent milps: finite-time feasibility and performance
  characterization,'' \emph{IEEE control systems letters}, vol.~2, no.~4, pp.
  563--568, 2018.

\bibitem{camisa2019milp}
A.~Camisa, I.~Notarnicola, and G.~Notarstefano, ``Distributed primal
  decomposition for large-scale {MILP}s,'' \emph{IEEE Transactions on Automatic
  Control}, pp. 1--1, 2021.

\bibitem{bertsekas1982constrained}
D.~P. Bertsekas, \emph{Constrained optimization and Lagrange multiplier
  methods}.\hskip 1em plus 0.5em minus 0.4em\relax Academic press, 1982.

\bibitem{testa2020choirbot}
A.~{Testa}, A.~{Camisa}, and G.~{Notarstefano}, ``Choi{R}bot: A {ROS} 2 toolbox
  for cooperative robotics,'' \emph{IEEE Robotics and Automation Letters},
  vol.~6, no.~2, pp. 2714--2720, 2021.

\bibitem{koenig2004design}
N.~Koenig and A.~Howard, ``Design and use paradigms for gazebo, an open-source
  multi-robot simulator,'' in \emph{2004 IEEE/RSJ International Conference on
  Intelligent Robots and Systems (IROS)(IEEE Cat. No. 04CH37566)},
  vol.~3.\hskip 1em plus 0.5em minus 0.4em\relax IEEE, 2004, pp. 2149--2154.

\bibitem{farina2019disropt}
F.~Farina, A.~Camisa, A.~Testa, I.~Notarnicola, and G.~Notarstefano, ``Disropt:
  a python framework for distributed optimization,'' \emph{IFAC-PapersOnLine},
  vol.~53, no.~2, pp. 2666--2671, 2020.

\bibitem{wilson2020robotarium}
S.~Wilson, P.~Glotfelter, L.~Wang, S.~Mayya, G.~Notomista, M.~Mote, and
  M.~Egerstedt, ``The robotarium: Globally impactful opportunities, challenges,
  and lessons learned in remote-access, distributed control of multirobot
  systems,'' \emph{IEEE Control Systems Magazine}, vol.~40, no.~1, pp. 26--44,
  2020.

\bibitem{cordeau2006branch}
J.-F. Cordeau, ``A branch-and-cut algorithm for the dial-a-ride problem,''
  \emph{Operations Research}, vol.~54, no.~3, pp. 573--586, 2006.

\bibitem{bemporad1999control}
A.~Bemporad and M.~Morari, ``Control of systems integrating logic, dynamics,
  and constraints,'' \emph{Automatica}, vol.~35, no.~3, pp. 407--427, 1999.

\bibitem{notarstefano2019distributed}
G.~Notarstefano, I.~Notarnicola, and A.~Camisa, ``Distributed optimization for
  smart cyber-physical networks,'' \emph{Foundations and
  Trends{\textregistered} in Systems and Control}, vol.~7, no.~3, pp. 253--383,
  2019.

\bibitem{silverman1972primal}
G.~J. Silverman, ``Primal decomposition of mathematical programs by resource
  allocation: {I}--basic theory and a direction-finding procedure,''
  \emph{Operations Research}, vol.~20, no.~1, pp. 58--74, 1972.

\end{thebibliography}
\end{document}